\documentclass{article}

\pdfoutput=1

\usepackage{microtype}
\usepackage{graphicx}
\usepackage{subfigure}
\usepackage{booktabs} 

\usepackage{hyperref}

\usepackage[accepted]{icml2023}

\usepackage{amsmath}
\usepackage{amssymb}
\usepackage{mathtools}
\usepackage{amsthm}

\usepackage[capitalize,noabbrev]{cleveref}

\theoremstyle{plain}
\newtheorem{theorem}{Theorem}[section]
\newtheorem{proposition}[theorem]{Proposition}
\newtheorem{lemma}[theorem]{Lemma}

\theoremstyle{definition}
\newtheorem{definition}[theorem]{Definition}

\theoremstyle{remark}
\newtheorem{remark}[theorem]{Remark}

\icmltitlerunning{Invariance in Policy Optimisation and Partial Identifiability in Reward Learning}

\usepackage{amsmath}
\usepackage{amssymb}
\usepackage{mathabx}
\usepackage{amsthm}
\usepackage{thmtools}
\usepackage{thm-restate}
\usepackage{hyperref}
\usepackage[noabbrev]{cleveref} 
\usepackage{url}
\usepackage{graphicx}
\usepackage{tikz}
\usetikzlibrary{automata,positioning,shapes}
\usepackage{enumitem}
\usepackage{thmtools} 
\usepackage{thm-restate}

\usepackage{styles/sidecomments}

\usepackage{xcolor}

\newcommand{\TrajDistributionBoltzmannWrt}[1]{\Delta^{#1}_\beta}
\newcommand{\TrajDistributionBoltzmann}{\Delta^\star_\beta}
\newcommand{\TrajDistributionMCE}{\Delta^{\mathrm{H}}_\beta}
\newcommand{\TrajDistributionOptimal}{\Delta_\star}

\newtheoremstyle{compactboldnormal}{3pt}{-3pt}{}{}{\bfseries}{.}{.5em}{}
\newtheoremstyle{compactbolditalic}{3pt}{-2pt}{\itshape}{}{\bfseries}{.}{.5em}{}
\newtheoremstyle{compactemphnormal}{3pt}{-2pt}{}{}{\itshape}{.}{.5em}{}

\newcommand{\States}{\mathcal{S}}
\newcommand{\Actions}{\mathcal{A}}
\newcommand{\reward}{R}
\newcommand{\TransitionDistribution}{\tau}
\newcommand{\InitStateDistribution}{\mu_0}
\newcommand{\discount}{\gamma}

\newcommand{\SxA}{{\States{\times}\Actions}}
\newcommand{\SxAxS}{{\States{\times}\Actions{\times}\States}}
\newcommand{\Q}{Q}
\newcommand{\V}{V}
\newcommand{\A}{A}
\newcommand{\Return}{G}
\newcommand{\Evaluation}{\mathcal{J}}
\newcommand{\Qfor}[1]{\Q_{#1}}
\newcommand{\Vfor}[1]{\V_{#1}}
\newcommand{\Afor}[1]{\A_{#1}}
\newcommand{\QStar}{\Q_\star}
\newcommand{\VStar}{\V_\star}
\newcommand{\AStar}{\A_\star}
\newcommand{\policy}{\pi}
\newcommand{\OptimalPolicy}{\policy_\star}
\newcommand{\basePolicy}{\policy_0}

\newcommand{\BoltzmannPolicyWrt}[1]{\policy^{#1}_{\beta}}
\newcommand{\BoltzmannRationalPolicy}{\policy^\star_\beta}

\newcommand{\cmpStarTraj}{\preceq_\star^\xi}
\newcommand{\cmpStarFrag}{\preceq_\star^\zeta}
\newcommand{\cmpSoftTraj}{\preceq_\beta^\xi}
\newcommand{\cmpSoftFrag}{\preceq_\beta^\zeta}
\newcommand{\refines}{\preceq}
\newcommand{\refinesStrict}{\prec}
\newcommand{\Mask}{\mathcal{X}}
\newcommand{\OPTFunc}{\mathcal{O}}
\newcommand{\argmax}{\mathrm{arg\,max}} 
\newcommand{\Powerset}{\mathcal{P}}
\newcommand{\Probability}{\mathbb{P}}
\newcommand{\Reals}{\mathbb{R}}
\newcommand{\CondExpect}[3]{\mathbb{E}_{#1}\bigl[ {#2} \bigm| {#3} \bigr]}
\newcommand{\Expect}[2]{\mathbb{E}_{#1}\bigl[ {#2} \bigr]}

\newcommand{\MCEPolicy}{\policy^{\mathrm{H}}_\beta}
\newcommand{\QSoft}{\Q^{\mathrm{H}}_\beta}
\newcommand{\QSoftN}[1]{\Q^{\mathrm{H}}_{\beta,#1}}
\newcommand{\cmpLotteryTraj}{\preceq_\mathcal{D}^\xi}

\begin{document}

\twocolumn[
\icmltitle{Invariance in Policy Optimisation and\\ Partial Identifiability in Reward Learning}

\icmlsetsymbol{equal}{*}

\begin{icmlauthorlist}
\icmlauthor{Joar Skalse}{equal,ox,fhi}
\icmlauthor{Matthew Farrugia-Roberts}{equal,uom}
\icmlauthor{Stuart Russell}{ucb}
\icmlauthor{Alessandro Abate}{ox}
\icmlauthor{Adam Gleave}{far}
\end{icmlauthorlist}

\icmlaffiliation{ox}{Department of Computer Science, Oxford University}
\icmlaffiliation{fhi}{Future of Humanity Institute, Oxford University}
\icmlaffiliation{ucb}{%
Center for Human-Compatible Artificial Intelligence,
University of California, Berkeley}
\icmlaffiliation{uom}{School of Computing and Information Systems, the University of Melbourne}
\icmlaffiliation{far}{FAR AI, Inc}

\icmlcorrespondingauthor{Joar Skalse}{joar.skalse@cs.ox.ac.uk}
\icmlcorrespondingauthor{Matthew Farrugia-Roberts}{matt.farrugia@unimelb.edu.au}

\icmlkeywords{Machine Learning, ICML}

\vskip 0.3in
]

\printAffiliationsAndNotice{\icmlEqualContribution} 

\begin{abstract}
It is often very challenging to manually design reward functions for complex, real-world tasks.
To solve this, one can instead use \emph{reward learning} to infer a reward function from data.
However, there are often multiple reward functions that fit the data equally well, even in the infinite-data limit.
This means that the reward function is only \emph{partially identifiable}.
In this work, we formally characterise the partial identifiability of the reward function
given several popular reward learning data sources, including expert demonstrations and trajectory comparisons.
We also analyse the impact of this partial identifiability for several downstream tasks,
such as policy optimisation.
We unify our results in a framework for comparing data sources and downstream tasks
by their invariances,
with implications for the design and selection of data sources for reward learning.
\end{abstract}

\section{Introduction}
\label{sec:intro}

Many problems can be represented as sequential
decision-making tasks, where the goal is to maximise a numerical \emph{reward function}
over several steps \citep{sutton2018}.
However, for real-world tasks, it is often challenging to design a reward function
that reliably incentivizes the right behaviour
\citep{amodei2016, leike2018, dulac2019}.
One approach to solving this problem is to use \emph{reward learning} algorithms, which aim to learn a reward function from data.
Reward learning algorithms can be based on many different data sources, including expert demonstrations~\citep{ng2000},
preferences over trajectories~\citep{christiano2017}, and many others~\citep{jeon2020}.
These algorithms can learn a reward function for tasks where it would be infeasible to specify a reward manually~\citep[e.g.,][]{abbeel2010, christiano2017, singh2019, stiennon2020}.

There will often be multiple reward functions that are consistent with a given data source, even
in the limit of infinite data.
For example, two reward functions may induce exactly the same expert behaviour; in that case, no amount of expert demonstrations can distinguish them. Similarly, two reward functions may induce exactly the same preferences over trajectories; in that case, no amount of preference data can distinguish them.
This means that the reward function is ambiguous, or \emph{partially identifiable}, based on these data sources.
For most data sources, this issue has been acknowledged, but its extent has not been characterised.
In this work, we formally characterise the ambiguity of the reward function for several popular data sources, including expert demonstrations (\S\ref{sec:results-expert}) and trajectory comparisons (\S\ref{sec:results-eval}).
Our results describe the infinite-data bounds for the information that can be recovered from these types of data.

Identifying a reward function uniquely is often unnecessary, because all plausible
reward functions might lead to the same outcome in a given application.
For example, if we want to learn a reward function in order to compute an optimal policy, then it is enough to learn a reward function that has the same optimal policies as the true reward function.
In general, ambiguity is not problematic if all compatible reward functions lead to identical downstream outcomes.
Therefore, we also characterise the \emph{ambiguity tolerance} for various applications (especially policy optimisation).
This allows us to evaluate whether or not the ambiguity of a given data source is problematic for a given application.

Ambiguity and ambiguity tolerance are formally related.
Both concern \emph{invariances} of objects that can be computed from reward functions to \emph{transformations} of those reward functions.
Thus, our main contribution is to catalogue the invariances of various
mathematical objects derived from the reward function.
In \cref{sec:lattice}, we explore a \emph{partial order} on these 
invariances and its implications for selecting and evaluating
data sources, addressing an open problem in reward learning~\citep[\S3.1]{leike2018}.

\subsection{Related Work}
\label{sec:related}

\emph{Inverse reinforcement learning}~\citep[IRL;][]{russell1998} is a central example of reward learning.
An IRL algorithm attempts to infer a reward function from demonstrations by an expert,
by inverting a model of the expert's \emph{planning
algorithm}~\citep{armstrong2017, shah2019}.
Existing work partially characterises the inherent ambiguity of
expert demonstrations for certain planning algorithms~\citep{ng2000, cao2021}
and classes of tasks~\citep{dvijotham2010, kim2021}.
We extend these results by considering a more expressive space of reward functions, more planning algorithms, and arbitrary stochastic tasks.

IRL is related to \emph{dynamic discrete choice}
\citep{rust1994, aguirregabiria2010},
a problem where identifiability has been studied extensively~\citep[e.g.,][]{aguirregabiria2005, srisuma2015, arcidiacono2020}.
We study a simpler setting with known tasks.
IRL also relates to \emph{preference elicitation}~\citep{rothkopf2011}
and \emph{inverse optimal control}~\citep{ab2020}.
Preferences over sequential trajectories are not typically considered as a
data source in these fields.

Reward learning methods have also been proposed for many other data 
sources~\citep{jeon2020}.
A popular and effective option is \emph{preferences over 
trajectories}~\citep{akrour2012, christiano2017}.
Unlike for IRL, the ambiguity arising from these data sources has
not been formally characterised in any previous work.
We contribute a formal characterisation of the ambiguity for central models
of evaluative feedback, including trajectory preferences.

Several studies have explored learning from a combination of both expert demonstrations and
preferences~\citep{ibarz2018, palan2019, biyik2020, koppol2020},
or other multimodal data sources~\citep{tung2018,krasheninnikov2021combining,jeon2020}.
One motivation is that different data sources may provide complementary reward 
information~\citep{koppol2020}, and therefore eliminate some ambiguity.
Similarly, \citet{amin2017} and \citet{cao2021} observe reduced ambiguity
by combining behavioural data across multiple \emph{tasks}.
We provide a general framework for understanding these results.

The primary application of a learnt
reward function is to compute an optimal policy~\citep{abbeel2004, wirth2017}.
\citet{ng1999} proved that \emph{potential-shaping transformations} always
preserve the set of optimal policies.
We extend this result, characterising the full set of transformations that
preserve optimal policies in each task, including for additional policy 
optimisation techniques such as maximum causal entropy reinforcement learning.

Ambiguity corresponds to the \emph{partial
identifiability}~\citep{lewbel2019} of the reward function modelled as a latent
parameter.
A common response to partial identifiability in reward learning has
been to impose additional constraints or assumptions until the data
identifies the reward function uniquely.
Following \citet{manski1995, manski2003} and \citet{tamer2010}, we instead
\emph{describe} ambiguity \emph{given} various constraints and assumptions.
This gives practitioners results that are appropriate for the data they in fact
have, and the ambiguity tolerance of their actual application.

\subsection{Preliminaries}
\label{sec:prelim}

We consider an idealised setting with finite, observable, infinite-horizon
sequential decision-making environments, formalised as \emph{Markov
Decision Processes} \citep[MDPs;][\S3]{sutton2018}.
An MDP is a tuple
$(\States, \Actions, \TransitionDistribution, \InitStateDistribution, 
\reward, \discount)$
where
  $\States$ and $\Actions$ are finite sets of environment \emph{states} and 
    agent \emph{actions};
  $\TransitionDistribution : \SxA\to\Delta(\States)$ encodes the 
    \emph{transition distributions} governing the environment dynamics;
  $\InitStateDistribution \in \Delta(\States)$ is an initial state
    distribution;
  $\reward : \SxAxS \rightarrow \Reals$ is a \emph{reward
    function};\footnotemark{}
and
  $\discount \in (0, 1)$ is a reward \emph{discount rate}.
We distinguish states in the support of $\mu_0$ as \emph{initial states}.
In some cases, we also wish to include \emph{terminal states}. These states have the property that $\TransitionDistribution(s | s, a) = 1$ and $\reward(s, a, s) = 0$ for all $a$.
A \textit{policy} is a function $\policy : \States \to \Delta(\Actions)$ that encodes
the behaviour of an agent in an MDP.

\footnotetext{%
    Notably, we consider deterministic reward functions that may depend on a
    transition's successor state.
    Alternative spaces of reward functions are often considered
    (such as functions from $\States$ or $\SxA$, or distributions).
    The chosen space has straightforward consequences for invariances, 
    which we discuss in \cref{apx:rewards}.
}

We represent the \emph{transition} from state~$s$ to state~$s'$ using
action~$a$ as the tuple $x = (s, a, s')$.
We classify $(s, a, s')$ as \emph{possible} in an MDP if
$s'$ is in the support of $\TransitionDistribution(s, a)$, otherwise it is 
\emph{impossible}.
A \emph{trajectory} is an infinite sequence of concatenated
transitions
$\xi = (s_0, a_0, s_1, a_1, s_2, \ldots)$, and
a \emph{trajectory fragment} of length $n$ is a finite sequence of $n$
concatenated transitions $\zeta = (s_0, a_0, s_1, \ldots, a_{n-1}, s_n)$.
A trajectory or trajectory fragment is \emph{possible} if all of its transitions are possible,
and is \emph{impossible} otherwise. 
A trajectory or trajectory fragment is \emph{initial} if its first state is initial.
We say that a state or transition is \emph{reachable} if it is part of some possible
and initial trajectory.

Given an MDP, we define the \emph{return function} $\Return$ as the
cumulative discounted reward of entire trajectories and trajectory
fragments:
$\Return(\zeta) = \sum_{t=0}^{|\zeta|-1} \discount^t \reward(s_t, a_t, s_{t+1})$ for a trajectory fragment $\zeta$ of length $|\zeta|$,
and similarly for trajectories.
A policy $\policy$ and a transition distribution $\TransitionDistribution$ together induce a distribution of trajectories starting from
each state.
We denote such a trajectory starting from $s$ with the random variable
$\Xi_s$, and its remaining components with random variables
$A_0, S_1, A_1, S_2$, and so on.

Given an MDP and a policy $\policy$,
the \emph{value function} encodes the expected return from a state,
$
\Vfor{\policy}(s)
=
\Expect
    {\Xi_s\sim\policy,\TransitionDistribution}
    {\Return(\Xi_s)}
$, and 
the \emph{$Q$-function} of $\policy$ encodes the expected return given an
initial action,
$
\Qfor{\policy}(s, a)
=
\CondExpect
    {\Xi_s\sim\policy,\TransitionDistribution}
    {\Return(\Xi_s)}
    {A_0=a}
$.
$\Qfor{\policy}$ and $\Vfor{\policy}$ satisfy a \emph{Bellman equation}:
\begin{align}
\label{eq:bellman-pi}
    \Qfor{\policy}(s, a)
    &=
    \Expect
        {S'\sim\TransitionDistribution(s, a)}
        {\reward(s, a, S') + \discount \Vfor{\policy}(S')},\\
    \Vfor{\policy}(s)
    &=
    \Expect{A\sim\policy(s)}{\Qfor{\policy}(s, A)},
\end{align}
for all $s \in \States$ and $a \in \Actions$.
Their difference, $\Afor{\policy}(s, a) = \Qfor{\policy}(s, a) - \Vfor{\policy}(s)$, is the
\emph{advantage function} of $\policy$.

We further define a \emph{policy evaluation function},
$\Evaluation$, encoding the expected return from following a particular policy
in an MDP,
$\Evaluation(\policy) = \Expect{S_0\sim\mu_0}{\Vfor{\policy}(S_0)}$.
$\Evaluation$ induces an order over policies.
A policy maximising $\Evaluation$ is an \emph{optimal policy}, denoted $\OptimalPolicy$.
Similarly, $\QStar$, $\VStar$, and $\AStar$ denote the $Q$-, value, and advantage
functions of an optimal policy.
Since $\Evaluation$ may be multimodal,
we often discuss the \emph{set} of optimal policies. However, $\QStar$, $\VStar$,
and $\AStar$ are each unique.

Besides optimal policies, we also consider policies resulting from alternative objectives.
Given an \emph{inverse temperature} parameter $\beta>0$, we define the
\emph{Boltzmann-rational} policy~\citep{ramachandran2007},
denoted $\BoltzmannRationalPolicy$, with a softmax distribution over the
optimal advantage function:
\begin{equation}
\label{eq:boltzmann-rational-policy}
\BoltzmannRationalPolicy (a \mid s)
= 
\frac{\exp\bigl(\beta \AStar(s, a)\bigr)}{\sum_{a' \in \Actions} \exp\bigl(\beta \AStar(s, a')\bigr)}.
\end{equation}
The \emph{Maximal Causal Entropy} (MCE)
policy~\citep{ziebart2010thesis,haarnoja2017} is given by
\begin{equation}\label{eq:mce_policy}
\MCEPolicy(a \mid s) = \frac{\exp\bigl(\beta \QSoft(s, a)\bigr)}{\sum_{a' \in \Actions} \exp\bigl(\beta \QSoft(s, a')\bigr)},
\end{equation}
where $\QSoft$ is the \emph{soft $Q$-function}, a regularised variant of the $Q$-function.
\citet[Theorem~2 and Appendix~A.2]{haarnoja2017} show that $\QSoft$ is the
unique function satisfying
\begin{equation}\label{eq:Q_soft}
\begin{split}
  \QSoft(s,a) = \mathbb{E}\Big[
    &R(s,a,S') +\\ 
    &\gamma \frac1\beta \log
        \sum_{a' \in \Actions} \exp\left( \beta\QSoft(S', a')\right)\Big].
\end{split}
\end{equation}
The MCE policy results from maximising a policy evaluation function
with an entropy regularisation term with weight
$\alpha = \beta^{-1}$~\citep{haarnoja2017}.
The Boltzmann-rational policy can also be connected to a kind of
(per-timestep) entropy regularisation~\citep{haarnoja2017}.

\section{Our Framework}
\label{sec:framework}

In this section, we describe a framework for analysing partial identifiability, that we will then use throughout the paper.
This framework, illustrated in \Cref{fig:framework}, makes it easy to compare different data sources.
The framework also simplifies reasoning about whether or not the ambiguity of a given data source could be problematic for a given application.

\subsection{The Reward Learning Lattice}\label{sec:lattice}

Given sets of states $\States$ and actions $\Actions$, let $\mathcal{R}$ be the set of reward functions definable over $\States$ and $\Actions$, that is, functions of type $\SxAxS \to \mathbb{R}$.
Moreover, let $X$ be some set of objects that can be computed from reward functions.

\begin{definition}
Given a function $f : \mathcal{R} \to X$, the \emph{invariance partition} of $f$ is the partition of $\mathcal{R}$ according to the equivalence relation $\sim$ where $R_1 \sim R_2$ if and only if $f(R_1) = f(R_2)$.
\end{definition}

Invariance partitions characterise the ambiguity of reward learning data sources, and also the ambiguity tolerance of different applications.
To see this, let us first build an abstract model of a reward learning algorithm.
Let $R^\star$ be the true reward function.
We model the data source as a function $f : \mathcal{R} \to X$, for some data space $X$,
so that the learning algorithm observes $f(R^\star)$.
Note that $f(R^\star)$ could be a distribution, which models the case where the data comprises a set of samples from some source, but it could also be some finite object.
A reasonable learning algorithm should converge to a reward function $R'$ that is compatible with the observed data, that is, such that $f(R') = f(R^\star)$.
This means that the invariance partition of $f$ groups together all reward functions that the learning algorithm could converge to.
When it comes to applications, let $g : \mathcal{R} \to Y$ be the function whose output we wish to compute. 
If $R^\star$ is the true reward function, then it is acceptable to instead learn a reward function $R'$ as long as $g(R') = g(R^\star)$.
This means that the invariance partition of $g$ groups together all reward functions that it would be acceptable to learn.
The information contained in the data source $f$ is guaranteed to be sufficient for computing $g$ if $f(R') = f(R^\star) \implies g(R') = g(R^\star)$.

To make this more intuitive, let us give an example. Consider first a reward learning data source, such as trajectory comparisons.
In this case, we can let $X$ be the set of all (strict, partial) orderings of the set of all trajectories, and $f$ be the function that returns the ordering of the trajectories that is induced by the trajectory return function, $G$. Let $R^\star$ be the true reward function. In the limit of infinite data, the reward learning algorithm will learn a reward function $R'$ that induces the same trajectory ordering as $R^\star$, which means that $f(R') = f(R^\star)$. Furthermore, if we want to use the learnt reward function to compute a policy, then we may consider the function $g : \mathcal{R} \to \Pi$ that takes a reward function $R$, and returns a policy $\pi^\star$ that is optimal under $R$ (in a given environment). Then as long as $f(R') = f(R^\star) \implies g(R') = g(R^\star)$, we will compute a policy that is optimal under the true reward $R^\star$.

\begin{figure}[!ht]
    \centering
    \tikzset{
    redbrick/.style={very thick,rounded corners=1pt,draw=red!40,fill=red!10},
    redspace/.style={very thick,draw=red!60,rounded corners=3pt},
    redspace full/.style={redspace,fill=red!10},
    obvspace/.style={very thick,rounded corners=20pt,draw=blue!50,fill=blue!15},
    outspace/.style={very thick,rounded corners=20pt,draw=green!50!blue!60,fill=green!50!blue!10},
    point/.style={draw,fill,circle,inner sep=1pt},
    minor point/.style={point,inner sep=0.5pt},
    fnarrow/.style={->,shorten >=5pt,shorten <=5pt,very thick},
    fntext/.style={midway,above,align=center},
}
\begin{tabular}{cl}
\begin{tikzpicture}
    \draw[white] (0.0, -1.2) -- (0.0, 2);
    \node at (0.0, 0.0) {\bf(a)};
\end{tikzpicture}
    &
\begin{tikzpicture}
    \path[obvspace] (-1.2, -1.2) rectangle (+1.2, +1.2);
    \path[redspace full] (2.8, -1.2) rectangle (5.2, +1.2);
    \node at ( 0.0, 1.5) {data space $X$};
    \node at ( 4.0, 1.5) {reward space $\mathcal{R}$};
    \draw[fnarrow] (1.3, 0.8) -- (2.7, 0.8)
        node[fntext] {reward\\[.55ex]learning\\[.46ex]};
    \node[point] (X) at (0.0, -.1) [label=left:$f(R)$] {};
    \node[point] at (3.7, -0.2) [label=above:$R?$] {};
    \node[point] at (4.2, -0.1) [label=above:$R?$] {};
    \draw (3.4,-0.4) -- (X.east)  -- (3.4,+0.4);
    \draw[rounded corners=1.5pt] (3.38,-0.42) rectangle (4.62,+0.42);
\end{tikzpicture}
\\
\begin{tikzpicture}
    \draw[white] (0.0, -1.2) -- (0.0, 2);
    \node at (0.0, 0.0) {\bf(b)};
\end{tikzpicture}
    &
\begin{tikzpicture}
    \begin{scope}
        \clip[rounded corners=2pt] (-1.2, -1.2) rectangle (1.2, +1.2);
        \path[redbrick] (-1.2, +0.9) rectangle (-0.6, +1.2);
        \path[redbrick] (-0.6, +0.9) rectangle ( 0.0, +1.2);
        \path[redbrick] ( 0.0, +0.9) rectangle ( 0.6, +1.2);
        \path[redbrick] ( 0.6, +0.9) rectangle ( 1.2, +1.2);
        \path[redbrick] (-1.5, +0.4) rectangle (-0.9, +0.9);
        \path[redbrick] (-0.9, +0.4) rectangle (-0.3, +0.9);
        \path[redbrick] (-0.3, +0.4) rectangle ( 0.3, +0.9);
        \path[redbrick] ( 0.3, +0.4) rectangle ( 0.9, +0.9);
        \path[redbrick] ( 0.9, +0.4) rectangle ( 1.5, +0.9);
        \path[redbrick] (-1.8, -0.4) rectangle (-0.6, +0.4);
        \path[redbrick] (-0.6, -0.4) rectangle ( 0.6, +0.4); 
        \path[redbrick] ( 0.6, -0.4) rectangle ( 1.8, +0.4);
        \path[redbrick] (-1.5, -0.4) rectangle (-0.9, -0.9);
        \path[redbrick] (-0.9, -0.4) rectangle (-0.3, -0.9);
        \path[redbrick] (-0.3, -0.4) rectangle ( 0.3, -0.9);
        \path[redbrick] ( 0.3, -0.4) rectangle ( 0.9, -0.9);
        \path[redbrick] ( 0.9, -0.4) rectangle ( 1.5, -0.9);
        \path[redbrick] (-1.2, -0.9) rectangle (-0.6, -1.2);
        \path[redbrick] (-0.6, -0.9) rectangle ( 0.0, -1.2);
        \path[redbrick] ( 0.0, -0.9) rectangle ( 0.6, -1.2);
        \path[redbrick] ( 0.6, -0.9) rectangle ( 1.2, -1.2);
    \end{scope}
    \path[redspace] (-1.2, -1.2) rectangle (1.2, +1.2);
    \path[obvspace] (4-1.2, -1.2) rectangle (4+1.2, +1.2);
    \node at ( 0, 1.5) {reward space $\mathcal{R}$};
    \node at ( 4, 1.5) {output space $X$};
    \draw[fnarrow] (1.3, 1.2) -- (4-1.3, 1.2) node[fntext] {$f$};
    \node[point] at (-.3, -0.1) [label=above:$R$] {};
    \node[point] at (0.2, -0.2) [label=above:$R'$] {};
    \node[point] (Z)  at (4.0, -.2) [label=above:{$f(R){=}f(R')$}] {};
    \node[minor point] (Z1) at (3.6, -.8) {};
    \node[minor point] (Z2) at (4.6, +.8) {};
    \draw[rounded corners=1.5pt] (-.62,-0.42) rectangle (.62,+0.42);
    \draw[rounded corners=1.5pt] (-0.62,-1.22) rectangle (0.02,-0.88);
    \draw[rounded corners=1.5pt] (0.88,0.92) rectangle (1.22,0.38);
    \draw (0.6,-0.4) -- (Z.west) -- (0.6,+0.4);
    \draw (0.0,-1.2) -- (Z1.west) -- (0.0,-0.9);
    \draw (1.2, 0.9) -- (Z2.west) -- (1.2, 0.4);
\end{tikzpicture}
    \\
\begin{tikzpicture}
    \draw[white] (0.0, -1.2) -- (0.0, 2.4);
    \node at (0.0, 0.0) {\bf(c)};
\end{tikzpicture}
    &
\begin{tikzpicture}
    \node[align=center] at ( 4.0, 1.7) {invariance\\[.5ex]partition of $f$};
    \begin{scope}
        \clip[rounded corners=2pt] (2.8, -1.2) rectangle (5.2, +1.2);
        \path[redbrick] (2.8, +0.9) rectangle (3.4, +1.2);
        \path[redbrick] (3.4, +0.9) rectangle (4.0, +1.2);
        \path[redbrick] (4.0, +0.9) rectangle (4.6, +1.2);
        \path[redbrick] (4.6, +0.9) rectangle (5.2, +1.2);
        \path[redbrick] (2.5, +0.4) rectangle (3.1, +0.9);
        \path[redbrick] (3.1, +0.4) rectangle (3.7, +0.9);
        \path[redbrick] (3.7, +0.4) rectangle (4.3, +0.9);
        \path[redbrick] (4.3, +0.4) rectangle (4.9, +0.9);
        \path[redbrick] (4.9, +0.4) rectangle (5.5, +0.9);
        \path[redbrick] (2.2, -0.4) rectangle (3.4, +0.4);
        \path[redbrick] (3.4, -0.4) rectangle (4.6, +0.4); 
        \path[redbrick] (4.6, -0.4) rectangle (5.8, +0.4);
        \path[redbrick] (2.5, -0.4) rectangle (3.1, -0.9);
        \path[redbrick] (3.1, -0.4) rectangle (3.7, -0.9);
        \path[redbrick] (3.7, -0.4) rectangle (4.3, -0.9);
        \path[redbrick] (4.3, -0.4) rectangle (4.9, -0.9);
        \path[redbrick] (4.9, -0.4) rectangle (5.5, -0.9);
        \path[redbrick] (2.8, -0.9) rectangle (3.4, -1.2);
        \path[redbrick] (3.4, -0.9) rectangle (4.0, -1.2);
        \path[redbrick] (4.0, -0.9) rectangle (4.6, -1.2);
        \path[redbrick] (4.6, -0.9) rectangle (5.2, -1.2);
    \end{scope}
    \path[redspace] (2.8, -1.2) rectangle (5.2, +1.2);
    
    \node[align=center] at ( 8.0, 1.7) {invariance\\[.5ex]partition of $g$};
    \begin{scope}
        \clip[rounded corners=2pt] (6.8, -1.2) rectangle (9.2, +1.2);
        \path[redbrick] (6.8, +1.2) -- (6.8, +0.4) -- (7.7, +0.4)
                     -- (7.7, +0.9) -- (7.4, +0.9) -- (7.4, +1.2)
                     -- (6.8, +1.2);
        \path[redbrick] (7.4, +1.2) -- (8.6, +1.2) -- (8.6, +0.9)
                     -- (8.3, +0.9) -- (8.3, +0.4) -- (7.7, +0.4)
                     -- (7.7, +0.9) -- (7.4, +0.9) -- (7.4, +1.2);
        \path[redbrick] (9.2, +1.2) -- (9.2, +0.4) -- (8.3, +0.4)
                     -- (8.3, +0.9) -- (8.6, +0.9) -- (8.6, +1.2)
                     -- (9.2, +1.2);
        \path[redbrick] (6.2, -0.4) rectangle (7.4, +0.4);
        \path[redbrick] (8.6, -0.4) rectangle (9.8, +0.4);
        \path[redbrick] (8.0, -0.9) -- (7.7, -0.9) -- (7.7, -0.4)
                     -- (7.4, -0.4) -- (7.4, +0.4) -- (8.6, +0.4)
                     -- (8.6, -0.4) -- (8.3, -0.4) -- (8.3, -0.9)
                     -- (8.0, -0.9);
        \path[redbrick] (6.8, -1.2) -- (6.8, -0.4) -- (7.7, -0.4)
                     -- (7.7, -0.9) -- (8.0, -0.9) -- (8.0, -1.2)
                     -- (6.8, -1.2);
        \path[redbrick] (9.2, -1.2) -- (9.2, -0.4) -- (8.3, -0.4)
                     -- (8.3, -0.9) -- (8.0, -0.9) -- (8.0, -1.2)
                     -- (9.2, -1.2);

    \end{scope}
    \path[redspace] (6.8, -1.2) rectangle (9.2, +1.2);

    \node at ( 6.0, 0.0) {\Huge$\preceq$};
    \node[align=center] at ( 6.0, 0.8) {partition\\[.5ex]refinemt.};
\end{tikzpicture}
\end{tabular}
    \caption{\label{fig:framework}%
        An overview of our framework.
        \textbf{(a)}~A reward learning algorithm infers a reward function from data,
            assuming the data has been generated from some \emph{data source},
            modelled as a function $f : \mathcal{R} \to X$.
            Depending on $f$, multiple reward functions $R \in \mathcal{R}$ may be
            consistent with the observed data $f(R) \in X$.
        \textbf{(b)}~By analysing the function $f$ one can partition the reward space
            into groups of reward functions that lead to the same output. We call
            this partition the \emph{invariance partition} of the function $f$.
            This partition can be effectively described by the set of transformations
            of the reward function that do not change the output of the function $f$.
        \textbf{(c)}~Such a partition characterises the ambiguity of reward learning
            based on a data source $f$. Moreover, it characterises the tolerance to
            ambiguity of computing $f(R)$ (where $R$ is a learnt reward function).
            We can compare data sources and applications by their partitions using
            the \emph{partition refinement} relation, that is, $f \refines g$
            if and only if
            $\forall R_1, R_2 \in \mathcal{R}, f(R_1) = f(R_2) \implies g(R_1) = g(R_2)$.
            If $f$ and $g$ are data sources, then $f \refines g$ means that $f$ contains
            no more ambiguity than $g$.
            If $g$ is a downstream application, then $f \refines g$ means $g$ can
            tolerate the ambiguity in a reward function learnt from data source $f$.
    }
\end{figure}
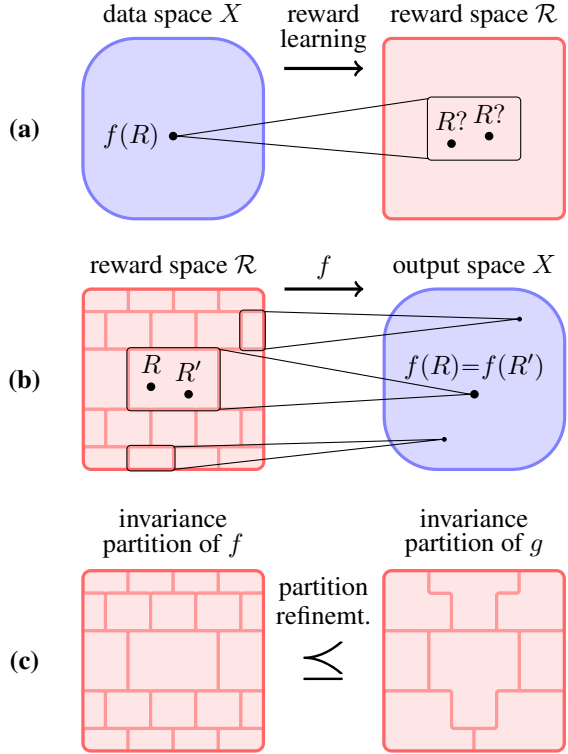

The characterisation of ambiguity and tolerance in terms of partitions of $\mathcal{R}$ suggests a natural \emph{partial order} on data sources and applications.
Namely, we use the \emph{partition refinement} relation on their invariances:

\begin{definition}\label{def:refinement}
Consider two functions $f : \mathcal{R} \to X$ and $g : \mathcal{R} \to Y$.
If $f(R_1) = f(R_2) \implies g(R_1) = g(R_2)$ for all $R_1, R_2 \in \mathcal{R}$, we write $f \refines g$
and say $f$ is \emph{no more ambiguous} than $g$.
If $f \refines g$ but not $g \refines f$, then we write $f \refinesStrict g$ and say that $f$ is \emph{strictly less ambiguous} than $g$.
\end{definition}

We can use this to formalise some important relationships. Given two reward learning data sources $f$ and $g$, if $f \refinesStrict g$ then we get strictly more information about the underlying reward function by observing data from $f$ than we get by observing data from $g$. Moreover, given a downstream application $h$, $f \refines h$ is precisely the condition of $h$ tolerating the ambiguity of the data source $f$. It is also worth noting that the invariances of an object are inherited by all objects which can be computed from it. Formally:

\begin{proposition}\label{prop:ambiguity_inherited}
Consider two functions $f : \mathcal{R} \to X$ and $g : \mathcal{R} \to Y$. If there exists a function $h : X \to Y$ such that $h \circ f = g$, then $f \refines g$.
\end{proposition}

This means that if there is an intermediate object (such as, for example, the $Q$-function) that is too ambiguous for a given application, then this will also hold for all objects that in principle could be computed from it.

Our framework places all reward learning data sources and applications in a lattice structure. In particular, the invariance partition of $f : \mathcal{R} \to X$ is a partition of $\mathcal{R}$, and the set of all these partitions forms a (bounded) lattice via the partition refinement relation. 
This, together with the properties of the refinement relation that we have just outlined, will make it simple and intuitive to reason about ambiguity.

\subsection{Reward Transformations}\label{sec:reward_transformations}

Throughout the rest of this paper, we will characterise the invariances of functions $f : \mathcal{R} \to X$ in terms of the transformations of $R$ that preserve $f(R)$. Formally:

\begin{definition}
\label{def:transformations_and_invariances}
A \emph{reward transformation} is a map $t : \mathcal{R} \to \mathcal{R}$.
We say that the \emph{invariances} of $f$ is a set of reward transformations $T$ if for all $R_1, R_2 \in \mathcal{R}$, we have that $f(R_1) = f(R_2)$ if and only if there is a $t \in T$ such that $t(R_1) = R_2$.
We then say that $f$ \emph{determines $R$ up to} $T$.
\end{definition}

When talking about a particular kind of object, we will for the sake of brevity usually leave the function $f$ implicit, and instead just mention the relevant object. 
For example, we might say that ``the Boltzmann-rational policy determines $R$ up to $T$''. This should be understood as saying that ``$f$ determines $R$ up to $T$, where $f$ is the function that takes a reward and returns the corresponding Boltzmann-rational policy''.
It is also worth noting that $f$ and $T$ often will be parameterised by $\TransitionDistribution$, $\InitStateDistribution$, or $\gamma$; this dependence will similarly not always be fully spelt out.

We will sometimes express the invariances of a function $f$ in terms of several sets of reward transformations -- for example, we might say that ``$f$ determines $R$ up to $T_1$ and $T_2$''. This should be understood as saying that $f$ determines $R$ up to $T$, where $T$ is the set of all transformations that can be formed by composing transformations in $T_1$ and $T_2$.

Our results are expressed in terms of several fundamental sets of reward transformations, that we will now define.
Before considering novel transformations, first recall \emph{potential shaping},
introduced by \citet{ng1999} and widely known to preserve optimal policies in all MDPs.
We explore some properties of potential shaping in \cref{apx:props}.
\begin{definition}[Potential Shaping]
A \emph{potential function} is a function $\Phi : \States \to \Reals$, where
$\Phi(s) = 0$ if $s$ is a terminal state.
If $\Phi(s) = k$ for all initial states then we say that $\Phi$ is
\emph{$k$-initial}. 
Let $\reward_1$ and $\reward_2$ be reward functions.
Given a discount $\discount$,
we say 
  $\reward_2$ is produced by \emph{($k$-initial) potential shaping} of $\reward_1$
if
  $\reward_2(s,a,s') = \reward_1(s,a,s') + \discount\cdot\Phi(s') - \Phi(s)$
  for some ($k$-initial) potential function $\Phi$.
\end{definition}

We next introduce a number of new transformations.

\begin{definition}[$S'$-Redistribution]
\label{def:sprime-redistribution}
Let $\reward_1$ and $\reward_2$ be reward functions.
Given transition dynamics $\TransitionDistribution$,
we say that
    $\reward_2$ is produced by \emph{$S'$-redistribution} of $\reward_1$
if
    $\Expect{S' \sim \TransitionDistribution(s,a)}{\reward_1(s,a,S')}
    = \Expect{S' \sim \TransitionDistribution(s,a)}{\reward_2(s,a,S')}$.
\end{definition}

$S'$-redistribution is simply any transformation that preserves $\Expect{S' \sim \TransitionDistribution(s,a)}{\reward(s,a,S')}$.\footnotemark{}
For example, if at least two states $s'_1$, $s'_2$ are in the support of
$\TransitionDistribution(s,a)$, then $S'$-redistribution could increase $\reward(s,a,s'_1)$ and decrease $\reward(s,a,s'_2)$, as long as $\Expect{S' \sim \TransitionDistribution(s,a)}{\reward(s,a,S')}$ stays constant.
Also note that $S'$-redistribution allows $\reward$ to be changed arbitrarily for impossible transitions.

\footnotetext{$S'$-redistribution depends crucially on the reward function's dependence on the successor state. This set of transformations collapses to the identity for simpler spaces of reward functions, as we explore in \cref{apx:rewards}.}

\begin{definition}[Monotonic Transformations]
Let $\reward_1$ and $\reward_2$ be reward functions. We say that $\reward_2$ is produced by a \emph{zero-preserving monotonic transformation} (ZPMT) of $\reward_1$ if for all $x$ and $x' \in \SxAxS$, $\reward_1(x) \leq \reward_1(x')$ if and only if $\reward_2(x) \leq \reward_2(x')$, and $\reward_1(x) = 0$ if and only if $\reward_2(x) = 0$.
Moreover, we say that $\reward_2$ is produced by \emph{positive linear scaling} of $\reward_1$ if $\reward_2 = c \cdot \reward_1$ for some positive constant $c$.
\end{definition}

A zero-preserving monotonic transformation is simply a monotonic transformation
that maps zero to itself. Positive linear scaling is a special
case.

\begin{definition}[Optimality-Preserving Transformation]
\label{def:opt-preserving}
Let $\reward_1$ and $\reward_2$ be reward functions.
Given
    transition dynamics $\TransitionDistribution$ and
    discount rate $\discount$,
we say 
    $\reward_2$ is produced by an
    \emph{optimality-preserving transformation} 
    of $\reward_1$
if
    there is a function $\Psi : \States \rightarrow \Reals$
    such that
        $\Expect
            {S' \sim \TransitionDistribution(s,a)}
            {\reward_2(s,a,S') + \discount \cdot \Psi(S')}
        \leq
            \Psi(s)
        $
        for all $s, a$,
        with equality if and only if $a \in \argmax_a A_{1\star}(s,a)$.
\end{definition}

Here $\Psi$ acts as a new value function, and the last condition ensures that $R_1$ and $R_2$ share the same optimal actions (and therefore the same optimal policies, cf.\ \cref{thm:optimal-policies-maximal}).

\begin{definition}[Masking]
Let $\reward_1$ and $\reward_2$ be reward functions.
Given a transition set $\Mask \subseteq \SxAxS$,
we say that
    $\reward_2$ is produced by a \emph{mask of $\Mask$} from $\reward_1$ 
if
    $\reward_1(s,a,s') = \reward_2(s,a,s')$ for all $(s,a,s') \notin \Mask$.
\end{definition}

Masking allows the reward to vary freely for some set of transitions $\Mask$. It is also worth noting that some of these sets of reward transformations are subsets of other sets. For example, all masks of impossible transitions are instances of $S'$-redistribution, and any potential shaping transformation is also optimality preserving, etc. All of these relationships are mapped out in Figure~\ref{fig:reward_transformation_subsumption}.

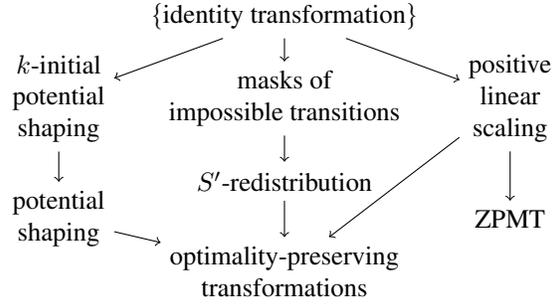
\begin{figure}
\centering
\begin{tikzpicture}
    \begin{scope}[every node/.style={align=center}]
    \node (0) at ( 4.0, 3.7) {$\{$identity transformation$\}$};
    \node (1) at ( 4.0, 2.6) {masks of\\impossible transitions};
    \node (2) at ( 4.0, 1.5) {$S'$-redistribution};
    \node (3) at ( 1.0, 2.6) {$k$-initial\\potential\\shaping};
    \node (4) at ( 1.0, 1.0) {potential\\shaping};
    \node (5) at ( 7.0, 2.6) {positive\\linear\\scaling};
    \node (6) at ( 7.0, 1.0) {ZPMT};
    \node (7) at ( 4.0, 0.3) {optimality-preserving\\transformations};
    \end{scope}
    \draw[->] (0) -- (1);
    \draw[->] (0) -- (3);
    \draw[->] (0) -- (5);
    \draw[->] (1) -- (2);
    \draw[->] (2) -- (7);
    \draw[->] (3) -- (4);
    \draw[->] (4) -- (7);
    \draw[->] (5) -- (6);
    \draw[->] (5) -- (7);
\end{tikzpicture}
\caption{\label{fig:reward_transformation_subsumption}%
    Subset relationships between the main sets of reward transformations
    (given a fixed transition distribution, initial state distribution, and discount rate).
    $A \to B$ denotes $A \subseteq B$.
    }
\end{figure}

\section{Invariances of Reward-Related Objects}
\label{sec:results}

In this section, we catalogue the \emph{invariances} of various important objects
that can be derived from reward functions.
As discussed in Section~\ref{sec:lattice}, these invariances describe both the ambiguity when these objects are used as data sources for reward learning, and the ambiguity tolerance when these objects are computed as part of an application.
Proofs and additional results are given in \cref{apx:directory,apx:props,apx:proofs}.

\subsection{Invariances of Policies}
\label{sec:results-expert}

We first characterise the invariances of different types of \emph{policies}. As discussed in \cref{sec:lattice}, this describes both the ambiguity of various inverse reinforcement learning algorithms, and the ambiguity tolerance of various applications in policy optimisation. We begin with the $Q$-function:

\begin{restatable}{theorem}{thmQFunction}
\label{lemma:Q_function_proximal_redistribution}
Given an MDP and a policy $\policy$,
the $Q$-function $\Qfor{\policy}$ of $\policy$ determines $\reward$ up to $S'$-redistribution.
The optimal $Q$-function, $\QStar$, and the soft $Q$-function $\QSoft$ (for any $\beta$),
both have precisely the same invariances.
\end{restatable}

As noted in \cref{prop:ambiguity_inherited}, this invariance is inherited by any object that can be derived from a $Q$-function.
Next, we turn our attention to optimal policies. We say that an optimal policy is \emph{maximally supportive} if it takes all optimal actions with positive probability.

\begin{restatable}{theorem}{thmOptimalPoliciesMaximal}
\label{thm:optimal-policies-maximal}
Given an MDP,
a maximally supportive optimal policy $\pi^\star$
determines $\reward$ up to
optimality-preserving transformations. The set of \emph{all} optimal policies has precisely the same invariances.
\end{restatable}

This theorem answers two questions at once. First, it tells us the exact ambiguity tolerance that we have if we want to use the learnt reward function to compute an optimal policy. Second, it also tells us what is the ambiguity of inverse reinforcement learning algorithms that observe an optimal policy, such as those developed by \citet{ng2000} and \citet{abbeel2004}.

There are many inverse reinforcement learning algorithms that do not assume that the observed demonstrator policy is optimal.
For example, \citet{ramachandran2007} and \citet{ziebart2008}
assume that the demonstrator policy is
\emph{Boltzmann-rational},
and \citet{ziebart2010paper} assume that it is \emph{causal entropy maximising}.
We catalogue the invariances of both these types of policies.
\begin{restatable}{theorem}{thmBoltzmannPoliciesBoth}
\label{thm:boltzmann-rational}
Given an MDP and an inverse temperature parameter $\beta$, the Boltzmann-rational policy $\BoltzmannRationalPolicy$ determines $\reward$ up to
$S'$-redistribution and potential shaping.
The MCE policy $\MCEPolicy$ has precisely the same invariances.
\end{restatable}

\subsection{Invariances of Trajectories}

Note that Theorem~\ref{thm:optimal-policies-maximal} and \ref{thm:boltzmann-rational} describe the invariances of \emph{policies}, rather than \emph{trajectories sampled from policies}.
The invariances of the trajectories are largely the same as of the policies themselves, but with some minor caveats.
In the infinite-data limit, trajectories sampled from a
policy reveal the \emph{distribution} of trajectories induced by the
policy, and therefore \emph{the policy itself} for all states reachable
via its supported actions.

Boltzmann-rational policies and MCE policies support \emph{all} actions, so the trajectory distribution determines the policy for
all \emph{reachable} states.
It follows that trajectory sampling introduces invariance solely to changes in the reward of unreachable transitions.

\begin{restatable}{theorem}{thmBoltzmannTrajectoriesBoth}
\label{thm:boltzmann-rational-trajectories}
Given an MDP and an inverse temperature parameter $\beta$,
the distribution of trajectories
$\TrajDistributionBoltzmann$ induced by the Boltzmann-rational policy 
$\BoltzmannRationalPolicy$,
or $\TrajDistributionMCE$ induced by the MCE policy $\MCEPolicy$,
determines $\reward$ up to
$S'$-redistribution, potential shaping, and a mask of unreachable
transitions.
\end{restatable}

Similarly, trajectories sampled from an optimal policy reveal the policy in those states that the policy visits. 
However, here there are some subtleties coming from the fact that an optimal policy may not visit all reachable states. This allows the reward to vary greatly, but not arbitrarily, in the states that are reachable but not visited by the optimal policy.
For this purpose, we introduce a more general notion of ``optimality-preserving transformation'' that is parameterised by a set-valued function $\OPTFunc$ that specifies the optimal actions.

\begin{definition}[General Optimality-Preserving Transformation]\label{def:general-opt-preserving}
Let $\reward$ and $\reward'$ be reward functions.
Given
    a function $\OPTFunc : \States \rightarrow \Powerset(\Actions) \setminus
    \{\emptyset\}$,
    transition dynamics $\TransitionDistribution$, and
    discount rate $\discount$,
we say 
    $\reward'$ is produced from $\reward$ by a
    \emph{(general) optimality-preserving transformation} with $\OPTFunc$
if
    there is a function $\Psi : \States \rightarrow \Reals$
    such that
        $\Expect
            {S' \sim \TransitionDistribution(s,a)}
            {\reward'(s,a,S') + \discount \Psi(S')}
        \leq
            \Psi(s)
        $
        for all $s, a$,
        with equality if and only if $a \in \OPTFunc(s)$.
\end{definition}

The difference between \cref{def:opt-preserving,def:general-opt-preserving} lies in the introduction of $\OPTFunc$.
In use, we constrain $\OPTFunc$ to be $\argmax_{a} \AStar(s, a)$ for $s$ in some subset of states (rather than for all states, which is the case for \cref{def:opt-preserving}).
If $\OPTFunc$ were unconstrained, the set would contain all possible transformations.
We can now state the invariances:

\begin{restatable}{theorem}{thmOptimalTrajectories}
\label{thm:optimal-trajectories}
Given an MDP, consider the distribution of trajectories induced by a maximally supportive optimal policy,
$\TrajDistributionOptimal$.
Let $\mathfrak{S}$ be the set of visited states.
Let $\mathfrak{O}$ be the set of functions $\OPTFunc$ defined on $\States$ such
that $\OPTFunc(s) = \argmax_{a} \AStar(s,a)$ for $s \in \mathfrak{S}$.
$\TrajDistributionOptimal$ determines $\reward$ up to
general optimality-preserving transformations for all $\OPTFunc \in \mathfrak{O}$.
\end{restatable}

Note that a mask of unreachable transitions is included as a subset of the transformations permitted by \cref{thm:optimal-trajectories}.
However, a mask of the complement of $\mathfrak{S}$ is \emph{not} included.
As $\OPTFunc$ is unconstrained outside $\mathfrak{S}$ the reward is effectively unconstrained in those states, except that the reward of transitions out of $\mathfrak{S}$ may have to ``compensate'' for the value of their successor states to prevent new actions that lead out of $\mathfrak{S}$ from becoming optimal.

\subsection{Invariances of Trajectory Evaluation}
\label{sec:results-eval}

The return function captures the reward accumulated over a trajectory, and is the basis for many reward learning algorithms.
In this section, we catalogue the invariances of the return function and related objects.

\begin{restatable}{theorem}{thmReturnTraj}
\label{thm:labels_of_trajectories}
Given an MDP,
the return function restricted to possible and initial trajectories, $\Return_\xi$,
determines $\reward$ up to
zero-initial potential shaping and a mask of \emph{unreachable} transitions.
\end{restatable}

For reward learning, observing $\Return_\xi$ would correspond to the case where the reward learning algorithm observes a trajectory or episode, together with the exact numerical value of the reward for that trajectory or episode.

A popular data source for reward learning is pairwise comparisons between trajectories, where a human is repeatedly shown two example behaviours, and asked to rank which of them is better \citep{akrour2012, christiano2017}.
It is common to model the comparisons as being based on the trajectory return, but with accompanying \emph{decision noise}, to account for the fact that the human labeller will sometimes make mistakes.
One option is to model this noise as following a \emph{Boltzmann distribution}, which says that the labeller is more likely to make a mistake if the trajectories have a similar value.
Formally, given an MDP and an inverse temperature $\beta > 0$,
let $\cmpSoftFrag$ be a distribution over each pair of
\emph{possible} trajectory fragments, $\zeta_1, \zeta_2$,
such that
$$
\Probability(\zeta_1 \cmpSoftFrag \zeta_2)
=
\frac
    {\exp(\beta\Return(\zeta_2))}
    {\exp(\beta\Return(\zeta_1)) + \exp(\beta\Return(\zeta_2))}
\,,
$$
and let $\cmpSoftTraj$ be the analogous distributions for possible, initial \emph{trajectories}.
Intuitively, the data source $\cmpSoftFrag$ corresponds to the case where a noisy expert ranks all trajectory fragments, and where the more valuable $\zeta_1$ is relative to $\zeta_2$, the more likely the expert is to select $\zeta_1$ over $\zeta_2$.

\begin{restatable}{theorem}{thmSoftComparisonsFrag}
\label{thm:boltzmann-comparisons-fragments}
Given an MDP, the distribution of comparisons of possible trajectory fragments,
$\cmpSoftFrag$,
determines $\reward$ up to
a mask of \emph{impossible} transitions.
\end{restatable}
\begin{restatable}{theorem}{thmSoftComparisonsTraj}
\label{thm:boltzmann-comparisons-trajectories}
Given an MDP,
the distribution of comparisons of possible and initial trajectories, $\cmpSoftTraj$,
determines $\reward$ up to
$k$-initial potential shaping and a mask of \emph{unreachable} transitions.
\end{restatable}

Note that the limited invariances of $\cmpSoftFrag$ arises from the very flexible comparisons permitted, including, for example, comparisons between individual transitions and empty trajectories.
It is also worth noting that these invariances rely very heavily
on the precise structure of the decision noise, and our assumption of infinite data.
First of all, in the infinite-data limit, each pair of trajectories is sampled multiple times, which means that we obtain the exact value of $\Probability(\zeta_1 \cmpSoftFrag \zeta_2)$ for each $\zeta_1$ and $\zeta_2$.
Moreover, since $\Probability(\zeta_1 \cmpSoftFrag \zeta_2)$ depends on the relative difference between the values of $\zeta_1$ and $\zeta_2$, it reveals cardinal information about the reward.
This is a property of Boltzmann noise that will not hold for many other kinds of decision noise.

It is also possible to model trajectory comparisons as noiseless comparisons based on the return of these trajectories.
The infinite-data limit then corresponds to the order induced by
the return function.
Formally, define the \emph{noiseless order of possible trajectory
fragments} as a relation, $\cmpStarFrag$, on possible trajectory fragments:
$$
    \zeta_1 \cmpStarFrag \zeta_2
    \Leftrightarrow
    \Return(\zeta_1) \leq \Return(\zeta_2)
\,.
$$
Analogously, define the \emph{noiseless order of possible and initial trajectories},
$\cmpStarTraj$, on possible, initial trajectories.
Here, the precise invariances will depend on the MDP. 

\begin{restatable}{theorem}{thmHardComparisonOfFragments}
\label{thm:noiseless-comparisons-trajectory-fragments}
We have the following bounds on the invariances of the noiseless order
of possible trajectory fragments, $\cmpStarFrag$.
In all MDPs:
\begin{enumerate}[label=(\arabic*),itemsep=0pt,topsep=0pt]
\item
    $\cmpStarFrag$ is invariant to positive linear scaling and
    a mask of impossible transitions; and
\item
    $\cmpStarFrag$ is not invariant to transformations other than
    zero-preserving monotonic transformations or masks of impossible 
    transitions.
\end{enumerate}
Moreover, there exist MDPs attaining each of these bounds.
\end{restatable}

Note that $\cmpSoftFrag$ is less ambiguous than $\cmpStarFrag$, even though $\cmpSoftFrag$ corresponds to noisy comparisons, and $\cmpStarFrag$ corresponds to noiseless comparisons. The reason for this is, again, that the Boltzmann noise reveals cardinal information about the reward function, whereas noiseless comparisons only reveal ordinal information.

We give a lower bound on the invariances of the noiseless order of possible
and initial trajectories, $\cmpStarTraj$.
Since $\cmpStarTraj$ can be derived from $\cmpSoftTraj$, it inherits the latter's
invariances.
Moreover, $\cmpStarTraj$ is always invariant to positive linear
scaling.

\begin{restatable}{theorem}{thmHardComparisonOfTrajectoriesBound}\label{thm:hard-trajectory-comparison}
Given an MDP, the noiseless order of possible and initial trajectories, $\cmpStarTraj$,
is invariant to
$k$-initial potential shaping,
positive linear scaling,
and 
a mask of unreachable transitions.
\end{restatable}

Note that Theorem~\ref{thm:hard-trajectory-comparison} gives a bound on the ambiguity of $\cmpStarTraj$, rather than an exact characterisation. Therefore, it does \emph{not} rule out additional invariances (unlike our other results).

We next give the transformations that preserve preferences over \emph{lotteries} (distributions) of trajectories. Formally, let $\cmpLotteryTraj$ be the relation on distributions over possible
initial trajectories given by
$$
\mathcal{D}_1 \cmpLotteryTraj \mathcal{D}_2
    \Leftrightarrow
    \Expect{\Xi \sim \mathcal{D}_1}{G(\Xi)}
    \leq
    \Expect{\Xi \sim \mathcal{D}_2}{G(\Xi)}.
$$
It is possible to model such preferences
as vNM-rational choices between lotteries over 
trajectories, with the return function $G$ as the utility function.
Such preferences are well known to be invariant to positive affine transformations of 
the utility function, and no other transformations~\citep[Appendix A]{vnm2}. 
We show that positive affine transformations of $G$ correspond to $k$-initial potential shaping and positive linear scaling of $R$.

\begin{restatable}{theorem}{thmLotteries}
\label{thm:comparisons-lotteries}
Given an MDP, $\cmpLotteryTraj$ determines $\reward$ up to
$k$-initial potential shaping,
positive linear scaling,
and 
a mask of unreachable transitions.
\end{restatable}

\subsection{Computing The Reward Learning Lattice}

As discussed in \cref{sec:lattice}, the reward invariances of all of these objects allow us to place them in a lattice structure. We show part of this lattice in \cref{fig:main}.
This allows us to represent many important relationships between data sources and applications in a graphical way.
In particular, recall that if $f \refines g$ then we get more information about the underlying reward function by observing data from $f$ than we get by observing data from $g$. Moreover, given some downstream application $h$, we have that $h$ tolerates the ambiguity of the data source $f$ exactly when $f \refines h$. Recall also that the invariances of an object are inherited by all objects which can be computed from it (\cref{prop:ambiguity_inherited}). This allows one to check whether a data source provides enough information for a given application based on whether the data source is above the application in the lattice.

\begin{figure}[!ht]
\centering
\begin{tikzpicture}[xscale=1.0]
    \begin{scope}[every node/.style={ellipse,fill=black!5,text=black!70}]
        \node[rotate=-7,inner sep=18,minimum width=100]
            at (2.2,1.4) {};
        \node[rotate=-35,inner sep=18,minimum width=100]
            at (5.1,1.8) {};
    \end{scope}
    \begin{scope}[every node/.style={text=black!70}]
        \node at (1.9,1.5) {\small pref.\ cmp.};
        \node at (5.4,1.6) {\small IRL};
    \end{scope}
    
    \node (R) at                ( 0.00, 3.9-0.0) {$\reward$}; 
    \node (ReturnFrag) at       ( 2.50, 3.9-0.1) {$\Return_\zeta$};
    \node (ReturnTraj) at       ( 2.50, 3.0-0.1) {$\Return_\xi$};
    \node (CmpSoftFrag) at      ( 1.00, 2.0-0.0) {$\cmpSoftFrag$};
    \node (CmpSoftTraj) at      ( 2.50, 2.0-0.1) {$\cmpSoftTraj$};
    \node (CmpLotteryTraj) at   ( 3.75, 1.2)     {$\cmpLotteryTraj$};
    \node (CmpStarFrag) at      ( 1.00, 1.1-0.0) {$\cmpStarFrag$};
    \node (CmpStarTraj) at      ( 2.50, 1.1-0.1) {$\cmpStarTraj$};
    \node (Q) at                ( 4.50, 3.9-0.225)
            {$\Qfor{\policy},\QStar,\QSoft$};
    \node (SoftPolicy) at       ( 4.50, 3.0-0.225)
            {$\BoltzmannRationalPolicy,\MCEPolicy$};
    \node (SoftTrajDist) at     ( 4.50, 2.0-0.225)
            {$\TrajDistributionBoltzmann,\TrajDistributionMCE$};
    \node (OptimalPolicy) at    ( 6.00, 2.1-0.325) {$\OptimalPolicy$};
    \node (OptimalPolicySet) at ( 7.50, 2.1-0.325) {$\{\OptimalPolicy\}$};
    \node (OptimalTrajDist) at  ( 6.00, 1.1-0.325) {$\TrajDistributionOptimal$};
    \draw[->] (R) -- (ReturnFrag);
    \draw[->] (Q) -- (SoftPolicy);
    \draw[->] (SoftPolicy) -- (OptimalPolicy);
    \draw[<->] (OptimalPolicy) -- (OptimalPolicySet);
    \draw[->] (SoftPolicy) -- (SoftTrajDist);
    \draw[->] (OptimalPolicy) -- (OptimalTrajDist);
    \draw[->] (SoftTrajDist) -- (OptimalTrajDist);
    \draw[->] (ReturnFrag) -- (Q);
    \draw[->] (ReturnFrag) -- (ReturnTraj);
    \draw[<->](ReturnFrag) -- (CmpSoftFrag);
    \draw[->] (ReturnTraj) -- (CmpSoftTraj);
    \draw[->] (CmpSoftFrag) -- (CmpSoftTraj);
    \draw[->] (CmpSoftFrag) -- (CmpStarFrag);
    \draw[->] (CmpSoftTraj) -- (CmpLotteryTraj);
    \draw[->] (CmpLotteryTraj) -- (CmpStarTraj);
    \draw[->] (CmpLotteryTraj) -- (OptimalTrajDist);
    \draw[->] (CmpSoftTraj) -- (SoftTrajDist);
    \draw[->] (CmpStarFrag) -- (CmpStarTraj);
\end{tikzpicture}
\caption{\label{fig:main}%
The invariances of objects that can be computed from a reward function in a fixed environment induce a partial order over these objects.
Here, a directed path from $X$ to $Y$ means that $X \refines Y$, which in turn implies that $X$ is no more ambiguous than $Y$, that $Y$ is tolerant to $X$'s ambiguity, and that $Y$ can be computed from $X$.
The objects are the reward function itself ($\reward$);
$Q$-functions ($\Qfor{\policy},\QStar,\QSoft$); Boltzmann-rational policies ($\BoltzmannRationalPolicy$), MCE policies ($\MCEPolicy$), maximally supportive optimal policies ($\OptimalPolicy$)  and their trajectory distributions ($\TrajDistributionBoltzmann,\TrajDistributionMCE,\TrajDistributionOptimal$);
the return function restricted to partial and full trajectories ($\Return_\zeta$, $\Return_\xi$);
Boltzmann-distributed ($\beta$) and noiseless ($\star$) comparisons between these trajectories ($\cmpSoftFrag$, $\cmpSoftTraj$, $\cmpStarFrag$, $\cmpStarTraj$) and lotteries over trajectories ($\cmpLotteryTraj$);
and the set of optimal policies ($\{\OptimalPolicy\}$).
}
\end{figure}
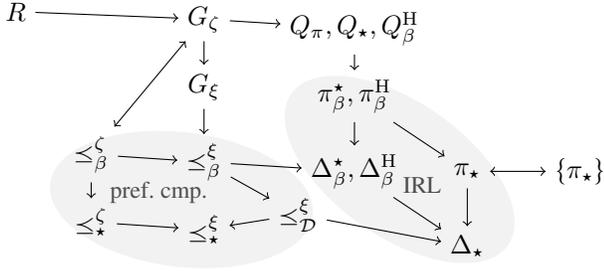

It is worth noting that the trajectories of optimal policies ($\Delta_\star$ in \cref{fig:main}) are more ambiguous than almost all other data sources, which also means that they tolerate the ambiguity of almost all other data sources.
In other words, the ambiguity of most data sources is unproblematic if the downstream task is to compute and then deploy an optimal policy.
However, this notably excludes noiseless comparison data (because this data only includes ordinal information, which is insufficient in stochastic environments).
Moreover, \cref{fig:main} assumes a fixed environment.
Applications in different environments may tolerate different ambiguity (\cref{sec:transfer}).

\section{Additional Results}\label{sec:more}

In this section, we provide some additional formal results, first concerning the case where we combine data from multiple different sources, and then concerning the case where we carry out reward learning in one MDP but use the learnt reward function in a different MDP.
We provide proofs of \cref{thm:complementary-ambiguity,thm:s-redistribution_and_transfer} in \cref{apx:proofs4}.

\subsection{Combining Data Sources}\label{sec:combining}

There has been recent interest in the prospect of combining information from multiple different sources for the purpose of reward learning \cite{krasheninnikov2021combining,jeon2020}.
Our results offer additional insights for this setting.
In particular, ambiguity refinement is a \emph{partial} order, with some data sources being \emph{incomparable}.
We note that such incomparable ambiguity is
\emph{complementary} ambiguity, in that by combining the associated data
sources, we can reduce the overall ambiguity about the latent reward.

\begin{restatable}{theorem}{thmComplementaryAmbiguity}\label{thm:complementary-ambiguity}
Given data sources $X$ and $Y$,
let $(X, Y)$ denote the combined data source formed from $X$ and $Y$.
If $X$ and $Y$ are incomparable,
then $(X, Y) \refinesStrict X$ and $(X, Y) \refinesStrict Y$.
\end{restatable}

This suggests that we can reduce the ambiguity of reward learning methods by using a mixture of data sources with complementary ambiguity.
Unfortunately, most data sources appear to have similar kinds of ambiguity, assuming a fixed MDP.
However, our results suggest that ambiguity could be reduced by incorporating data from \emph{multiple} MDPs,
along the lines of \citet{amin2017} and \citet{cao2021}.

\subsection{Transfer Learning}\label{sec:transfer}

It is interesting to consider the setting where the reward is learnt in one MDP, but used in a different MDP.
This captures, for example, the common sim-to-real setting, where learning occurs in a simulation whose dynamics differ slightly from those of the real environment. It also models the case where training data for the reward learning algorithm is not available for the deployment environment, but is available for a different environment.

We can begin by noting that no guarantees can be obtained if the two MDPs have very different reachable state spaces; in that case, the learnt reward function will mostly depend on the inductive bias of the learning algorithm. 
However, we will demonstrate that similar problems can occur even with much smaller differences between the learning environment and the deployment environment:

\begin{restatable}{theorem}{thmSRedistributionAndTransfer}
\label{thm:s-redistribution_and_transfer}
Let $\mathcal{L} : \SxA \to \Reals$ be any function, $R_1$ any reward function, and $\TransitionDistribution_1, \TransitionDistribution_2$ any transition distributions.
Then there exists a reward function $\reward_2$, produced from $\reward_1$ by $S'$-redistribution under $\TransitionDistribution_1$,
such that 
    $\Expect{S' \sim \TransitionDistribution_2(s,a)}{R_2(s,a,S')} = \mathcal{L}(s,a)$
for all $s,a$ such that $\TransitionDistribution_1(s,a) \neq \TransitionDistribution_2(s,a)$.
\end{restatable}

To unpack this, let $R_1$ be the true reward function, $\TransitionDistribution_1$ be the transition dynamics of the training environment, and $\TransitionDistribution_2$ be the transition dynamics of the deployment environment. Theorem~\ref{lemma:Q_function_proximal_redistribution} then says, roughly, that if the reward learning algorithm is invariant to $S'$-redistribution, and $\TransitionDistribution_1$ and $\TransitionDistribution_2$ differ for enough states, then the learnt reward function is essentially unconstrained in the deployment environment, meaning that no guarantees could be obtained.
Moreover, note that \cref{lemma:Q_function_proximal_redistribution} and \cref{prop:ambiguity_inherited} imply that this result extends to \emph{any} object that can be computed from a $Q$-function, which is a very broad class.
\Cref{thm:s-redistribution_and_transfer} then suggests that any such data source is too ambiguous to guarantee transfer to a different environment.
Note that this strong result relies on the formulation of rewards as depending
on the successor state (see \cref{apx:rewards}).

\section{Discussion}

In this section, we will discuss the implications of our results, and some limitations and directions for future work.

\subsection{Limitations and Future Work}
\label{sec:future-work}

We have catalogued the invariances of many important objects. However, there are some interesting objects that are not covered by our analysis. To start with, \cref{thm:hard-trajectory-comparison} only gives a bound on the invariance partition of noiseless trajectory comparisons, rather than an exact characterisation. It would be interesting to derive this invariance partition exactly.
It would also be interesting to characterise what happens if various restrictions are imposed on the trajectories and trajectory fragments considered throughout \cref{sec:results-eval}, such as, for example, a minimum and maximum length of the trajectories.
Moreover, future work could characterise the ambiguity of the \emph{policy ordering} induced by the policy evaluation function.
\Cref{thm:optimal-policies-maximal} characterises ambiguity tolerance of computing an optimal policy. However, in practice, one often uses a reinforcement learning algorithm that is not guaranteed to find a globally optimal policy, so it may be prudent to preserve the entire policy order, rather than just the set of maximising policies.

Moreover, our results primarily concern the asymptotic behaviour of reward learning algorithms, in the limit of infinite data.
In practice, data sets are finite and, when data collection is expensive, may
be quite small. 
An important direction for future work is to characterise how much information is contained in data sets of varying sizes and data sources.
This would enable practitioners to determine the most sample-efficient data source.

Furthermore, our results rely on the assumption that data in fact is generated according to the process that is assumed by the reward learning algorithm.
However, in reality, it is unlikely that these assumptions will hold perfectly \citep{orsini2021}. 
For example, human demonstrations are rarely perfectly optimal or Boltzmann-rational.
This means that the learning algorithms are often \emph{misspecified} in practice. Studying this setting formally is another important direction for future work.

\subsection{Conclusion and Implications}

We have precisely characterised the ambiguity of many reward learning data sources, and the ambiguity tolerance of many applications using reward functions. 
These results can be used to easily answer whether or not a given reward learning data source is appropriate for a given application, and to compare different reward learning data sources against each other.
This is valuable to help practitioners choose what reward learning method to use for a given problem.

Our results are also helpful for developing new reward learning algorithms.
In developing a reward learning algorithm, it is relevant to ask how effective this algorithm is relative to an optimal algorithm for the given data source.
Our results characterise the information that is available in each data source, which makes it possible to evaluate a reward learning algorithm against this theoretical upper bound.

Our framework enables direct comparisons between different data sources.
We find that some data sources are strictly less informative than others.
For example, noiseless preference comparisons are strictly less informative than return labels.
By contrast, other data sources are incomparable and have complementary ambiguity.
For example, the ambiguity of the $Q$-values is incomparable to the ambiguity of episode return values (the former are invariant to $S'$-redistribution but not to potential shaping, whereas the latter are invariant to some potential shaping but not to $S'$-redistribution).
As discussed in \cref{sec:combining}, these results can be used to identify promising opportunities for combining information from multiple sources.

We find that in a fixed environment, many reward learning data sources are sufficient for the purpose of computing an optimal policy, including labelled trajectories, Boltzmann comparisons of trajectories and trajectory fragments, and trajectories sampled from any of the three types of policies that are common in IRL (optimal, Boltzmann-rational, or causal entropy maximising).
However, this notably excludes noiseless comparisons between trajectory fragments
in some MDPs, since zero-preserving monotonic transformations are not, in general, optimality preserving.
These relationships are summarised in Figure~\ref{fig:main}.
Theorem~\ref{thm:s-redistribution_and_transfer} also shows that a very wide range of reward learning data sources will be too ambiguous to facilitate transfer to new environments.

\section*{Acknowledgements}

JS, MFR, and AG contributed to this research while affiliated with the Center for Human-Compatible Artificial Intelligence, UC Berkeley.

The authors thank Daniel Filan, Erik Jenner, Cassidy Laidlaw, and Smitha Milli for feedback on earlier versions of the manuscript and Daniel Murfet for insightful discussions about the project.
We also thank our anonymous reviewers for suggestions which have helped to improve the clarity of our results and notation.

\bibliography{references}
\bibliographystyle{icml2023}

\newpage
\appendix
\onecolumn
\section{Directory of Results}
\label{apx:directory}

\newcommand{\TI}{\ensuremath{{\equiv}}}
\newcommand{\Ti}{\ensuremath{\TI^{*}}}
\newcommand{\Tx}{\ensuremath{{\times}}}
\newcommand{\TIx}{\Tx/\TI}
\newcommand{\transformationLabel}[1]{%
    \rotatebox{50}{\parbox[c][1.24em]{5em}{%
        \linespread{0.8}\small\selectfont\centering #1}}}
\newcommand{\objectDescription}[1]{\small\raggedright\linespread{0.9} #1}
\newcommand{\prooflink}[1]{%
        \textsuperscript{\hyperlink{#1}{\qedsymbol}}}
\renewcommand{\arraystretch}{1.13} 

\begin{table}[ht!]
    \caption{%
        An overview of all invariance results.
        Symbols:
        (\TI) The object is invariant to the transformations;
        (\Ti) The object is invariant to the transformations as a special
        case of the invariances listed in the theorem;
        (\Tx) The object is not generally invariant to the transformations
        (more precisely, it is invariant only to those that can be represented
        as a combination of the listed transformations);
        (\TIx) The extent of the object's invariance to the transformations
        depends on the MDP
        (as in \Cref{thm:labels_of_trajectories}).
        Blank cells indicate unresolved invariances (as in
        \Cref{thm:hard-trajectory-comparison}, see
        also \Cref{remark:noiseless-comparisons-trajectories-bound}).
    }
    \label{table:directory}

    \begin{center}
    \begin{tabular}{cp{9.5em}cc@{\hskip -1.5em}c@{\hskip -1.5em}c@{\hskip -1.5em}c@{\hskip -1.5em}c@{\hskip -1.5em}c@{\hskip -1.5em}c@{\hskip -1.5em}c@{\hskip -1.5em}c@{\hskip -1.5em}c@{\hskip -1.5em}c@{}}
    \toprule
        &&& \multicolumn{11}{c}{Class of transformations}
    \\ \cmidrule(l){4-14}
        \multicolumn{2}{c}{Reward-derived object}
        & \parbox[b]{2.2em}{\centering Thm. \& proof link}
        & \transformationLabel{Identity}
        & \transformationLabel{0-initial $\Phi$ shaping}
        & \transformationLabel{$k$-initial $\Phi$ shaping}
        & \transformationLabel{$\Phi$ shaping}
        & \transformationLabel{S'-redistr.}
        & \transformationLabel{Pos.\ lin.\ scaling}
        & \transformationLabel{0-preserving monotonic}
        & \transformationLabel{Optimality-preserving}
        & \transformationLabel{Gen. opt.-preserving}
        & \transformationLabel{Impossible mask}
        & \transformationLabel{Unreachable mask}
    \\ \midrule
        $\reward$
            & \objectDescription{Reward function}
            & --
            & \TI & \Tx & \Tx & \Tx & \Tx & \Tx & \Tx & \Tx & \Tx & \Tx & \Tx
    \\
        $\Qfor{\policy}$
            & \objectDescription{$Q$-function for policy $\policy$}
            & \ref{lemma:Q_function_proximal_redistribution}%
                \prooflink{proof:q}
            & \Ti & \Tx & \Tx & \Tx & \TI & \Tx & \Tx & \Tx & \Tx & \Ti & \Tx
    \\
        $\QStar$
            & \objectDescription{Optimal $Q$-function}
            & \ref{lemma:Q_function_proximal_redistribution}%
                \prooflink{proof:q}
            & \Ti & \Tx & \Tx & \Tx & \TI & \Tx & \Tx & \Tx & \Tx & \Ti & \Tx
    \\
        $\QSoft$
            & \objectDescription{Soft $Q$-function}
            & \ref{lemma:Q_function_proximal_redistribution}%
                \prooflink{proof:q}
            & \Ti & \Tx & \Tx & \Tx & \TI & \Tx & \Tx & \Tx & \Tx & \Ti & \Tx
    \\
        $\BoltzmannRationalPolicy$
            & \objectDescription{Boltzmann-rational policy}
            & \ref{thm:boltzmann-rational}%
                \prooflink{proof:boltzmann-rational-policy}
            & \Ti & \Ti & \Ti & \TI & \TI & \Tx & \Tx & \Tx & \Tx & \Ti & \Tx
    \\
        $\MCEPolicy$
            & \objectDescription{MCE policy}
            & \ref{thm:boltzmann-rational}%
                \prooflink{proof:mce-policy}
            & \Ti & \Ti & \Ti & \TI & \TI & \Tx & \Tx & \Tx & \Tx & \Ti & \Tx
    \\
        $\OptimalPolicy$
            & \objectDescription{Maximally supportive optimal policy}
            & \ref{thm:optimal-policies-maximal}%
                \prooflink{proof:supportive-optimal-policies}
            & \Ti & \Ti & \Ti & \Ti & \Ti & \Ti & \Tx & \TI & \Tx & \Ti & \Tx
    \\
        $\TrajDistributionBoltzmann$
            & \objectDescription{Trajectory distribution induced by
                $\BoltzmannRationalPolicy$}
            & \ref{thm:boltzmann-rational-trajectories}%
                \prooflink{proof:boltzmann-rational-trajectories}
            & \Ti & \Ti & \Ti & \TI & \TI & \Tx & \Tx & \Tx & \Tx & \Ti & \TI
    \\
        $\TrajDistributionMCE$
            & \objectDescription{Trajectory distribution induced by
                $\MCEPolicy$}
            & \ref{thm:boltzmann-rational-trajectories}%
                \prooflink{proof:mce-trajectories}
            & \Ti & \Ti & \Ti & \TI & \TI & \Tx & \Tx & \Tx & \Tx & \Ti & \TI
    \\
        $\TrajDistributionOptimal$
            & \objectDescription{Trajectory distribution induced by
            $\OptimalPolicy$}
            & \ref{thm:optimal-trajectories}%
                \prooflink{proof:supportive-optimal-trajectories}
            & \Ti & \Ti & \Ti & \Ti & \Ti & \Ti & \Tx & \Ti & \TI & \Ti & \Ti
    \\
        $\Return_\zeta$
            & \objectDescription{Return on possible trajectory
                fragments}
            & \ref{thm:labels_of_fragments}%
              \prooflink{proof:return-fragments}
            & \Ti & \Tx & \Tx & \Tx & \Tx & \Tx & \Tx & \Tx & \Tx & \TI & \Tx
    \\
        $\Return_\xi$
            & \objectDescription{Return on possible, initial
                trajectories}
            & \ref{thm:labels_of_trajectories}%
                \prooflink{proof:return-trajectories}
            & \Ti & \TI & \Tx & \Tx & \Tx & \Tx & \Tx & \Tx & \Tx & \Ti & \TI
    \\
        $\cmpSoftFrag$
            & \objectDescription{Boltzmann comparisons of possible
                fragments}
            & \ref{thm:boltzmann-comparisons-fragments}%
                \prooflink{proof:soft-compare-fragments}
            & \Ti & \Tx & \Tx & \Tx & \Tx & \Tx & \Tx & \Tx & \Tx & \TI & \Tx
    \\
        $\cmpSoftTraj$
            & \objectDescription{Boltzmann comparisons of possible
                initial trajectories}
            & \ref{thm:boltzmann-comparisons-trajectories}%
                \prooflink{proof:soft-compare-trajectories}
            & \Ti & \Ti & \TI & \Tx & \Tx & \Tx & \Tx & \Tx & \Tx & \Ti & \TI
    \\
        $\cmpStarFrag$
            & \objectDescription{Noiseless order of possible fragments}
            & \ref{thm:noiseless-comparisons-trajectory-fragments}%
                \prooflink{proof:hard-compare-fragments}
            & \Ti & \Tx & \Tx & \Tx & \Tx & \TI &\TIx & \Tx & \Tx & \TI & \Tx
    \\
        $\cmpStarTraj$
            & \objectDescription{Noiseless order of possible initial
            trajectories}
            & \ref{thm:hard-trajectory-comparison}%
                \prooflink{proof:hard-compare-trajectories-bound}
            & \Ti & \Ti & \TI &     &     & \TI &     &     &     & \Ti & \TI
    \\
        $\cmpLotteryTraj$
            & \objectDescription{Order of distributions over
                possible initial trajectories}
            & \ref{thm:comparisons-lotteries}%
                \prooflink{proof:compare-lotteries}
            & \Ti & \Ti & \TI & \Tx & \Tx & \TI & \Tx & \Tx & \Tx & \Ti & \TI
    \\
        $\{\OptimalPolicy\}$
            & \objectDescription{Set of optimal policies}
            & \ref{thm:optimal-policies-maximal}%
                \prooflink{proof:optimal-policies-maximal}
            & \Ti & \Ti & \Ti & \Ti & \Ti & \Ti & \Tx & \TI & \Tx & \Ti & \Tx
    \\ \midrule
        $\Afor{\policy}$
            & \objectDescription{Advantage function for arbitrary policy
                $\pi$}
            & \ref{lemma:advantage}
            & \Ti & \Ti & \Ti & \TI & \TI & \Tx & \Tx & \Tx & \Tx & \Ti & \Tx
    \\
        $\AStar$
            & \objectDescription{Optimal advantage function}
            & \ref{lemma:advantage}
            & \Ti & \Ti & \Ti & \TI & \TI & \Tx & \Tx & \Tx & \Tx & \Ti & \Tx
    \\
        $\BoltzmannPolicyWrt{\basePolicy}$
            & \objectDescription{Boltzmann policy for arbitrary base policy
                $\basePolicy$}
            & \ref{lemma:boltzmann-general}
            & \Ti & \Ti & \Ti & \TI & \TI & \Tx & \Tx & \Tx & \Tx & \Ti & \Tx
    \\
        $\TrajDistributionBoltzmannWrt{\basePolicy}$
            & \objectDescription{Trajectory distribution induced by
                $\BoltzmannPolicyWrt{\basePolicy}$}
            & \ref{lemma:boltzmann-trajectories-general}
            & \Ti & \Ti & \Ti & \TI & \TI & \Tx & \Tx & \Tx & \Tx & \Ti & \TI
    \\
        \bottomrule
    \end{tabular}
    \end{center}
\end{table}

\section{Properties of Fundamental Reward Transformations}
\label{apx:props}

We begin with some supporting results concerning the basic reward
transformations used in \cref{sec:results} to characterise the invariances
of various objects derived from the reward function.

The following result captures how potential shaping affects
various reward-related functions.

\begin{restatable}{lemma}{thmPropertiesOfPotentials}
\label{lemma:change_from_potentials}
\label{cor:potentials_to_policies}
Consider $M$ and $M'$, two MDPs differing only in their reward functions,
respectively $\reward$ and $\reward'$.
Denote the return function, $Q$-function, value function,
policy evaluation function, and advantage function of $M'$ by
$\Return'$, $\Qfor{\policy}'$, $\Vfor{\policy}'$, $\Evaluation'$, and
$\Afor{\policy}'$.
If $\reward'$ is produced by potential shaping of $\reward$ with a potential function $\Phi$, then:
\begin{enumerate}[label=(\arabic*),topsep=0pt,itemsep=0pt]
\item
    for a trajectory fragment $\zeta = (s_0, a_0, s_1, \ldots, s_n)$,
    $\Return'(\zeta) = \Return(\zeta) + \discount^n\Phi(s_n) - \Phi(s_0)$;
\item
    for a trajectory $\xi = (s_0, a_0, \ldots)$,
    $\Return'(\xi) = \Return(\xi) - \Phi(s_0)$;
\item
    for a state $s \in \States$ and action $a \in \Actions$,
    $\Qfor{\policy}'(s, a) = \Qfor{\policy}(s, a) - \Phi(s)$;
\item
    for a state $s \in \States$, 
    $\Vfor{\policy}'(s) = \Vfor{\policy}(s) - \Phi(s)$;
\item
    for a policy $\policy$,
    $\Evaluation'(\policy)
      = \Evaluation(\policy)
        - \Expect{S_0 \sim \InitStateDistribution}{\Phi(S_0)}
    $; and
\item
    for a state $s \in \States$, and action $a \in \Actions$,
    $\Afor{\policy}'(s, a) = \Afor{\policy}(s, a)$.
\end{enumerate}
\end{restatable}
\begin{proof}
(1) is given by a straightforward telescopic argument.
For (2), take the limit as the length of a prefix goes to infinity,
whereupon $\discount^n \Phi(s_n)$ goes to zero
($\discount < 1$ by definition, and $\Phi(s_n)$ is bounded since its domain is finite).
(3) and (4) were proved for optimal policies by \citet{ng1999}, and
they also observed that the extension to arbitrary policies is
straightforward (it follows immediately from~(2), for example).
(5)~is immediate from (4).
(6)~follows from (3) and (4) as the shifts of $-\Phi(s)$ to both the
$Q$- and value functions cancel eachother.
\end{proof}

In the next result we show that potential shaping induces a similar
state-dependent shift in the soft $Q$-function as well.

\begin{restatable}{lemma}{thmPropertiesOfPotentialsSoftQ}
\label{lemma:potentials_to_soft_Q}
Consider $M_1$ and $M_2$, two MDPs differing only in their reward functions,
respectively $\reward_1$ and $\reward_2$.
Denote the soft $Q$-function of $M_1$ by $\QSoftN{1}$, and of $M_2$ by
$\QSoftN{2}$.
If $\reward_2$ is produced by potential shaping of $\reward_1$ with
a potential function $\Phi$, then for all states $s \in \States$ and actions
$a \in \Actions$, $\QSoftN{2}(s, a) = \QSoftN{1}(s, a) - \Phi(s)$.
\end{restatable}
\begin{proof}
We will appeal to to uniqueness of the soft $Q$-function.
By definition,
$\reward_2(s, a, s') = \reward_1(s, a, s') + \gamma\cdot\Phi(s') - \Phi(s)$.
Combining with \cref{eq:Q_soft}, we have for all
$s\in\States$ and $a\in\Actions$:
\begin{align*}
    \QSoftN{1}(s,a)
    &= \Expect
        {S' \sim \TransitionDistribution(s, a)}{
            \reward_1(s,a,S')
            + \gamma \frac1\beta
                \log \sum_{a'\in\Actions} \exp\beta\QSoftN{1}(S', a')
        }
    \\
    &= \Expect
        {S' \sim \TransitionDistribution(s, a)}{
            \reward_2(s, a, S') - \gamma\cdot\Phi(S') + \Phi(s)
            + \gamma \frac1\beta
                \log \sum_{a'\in\Actions} \exp\beta\QSoftN{1}(S', a')
        }
    \\
    \rightarrow \quad
    \QSoftN{1}(s,a) - \Phi(s)
    &= \Expect
        {S' \sim \TransitionDistribution(s, a)}{
            \reward_2(s, a, S')
            + \gamma \frac1\beta
                \log \sum_{a'\in\Actions} \exp\beta
                    \left(\QSoftN{1}(S', a') - \Phi(S')\right)
        }\,.
\end{align*}
We see that $\QSoftN{1}(s,a) - \Phi(s)$ satisfies \cref{eq:Q_soft}
for $\QSoftN{2}(s, a)$, for all $s\in\States$ and $a\in\Actions$.
Since the soft $Q$-function is the unique solution to this equation, we
conclude
    $\QSoftN{2}(s, a) = \QSoftN{1}(s, a) - \Phi(s)$.   
\end{proof}

We next show that $k$-initial potential shaping
and linear scaling correspond to affine transformations of
$\Return$.
\begin{restatable}{lemma}{thmPotentialsAndShifting}
\label{lemma:potentials_and_episodes}
Let
    $(\States,\Actions,\TransitionDistribution,\InitStateDistribution,\reward,
    \discount)$ be an MDP,
    $\reward'$ a reward function, and
    $k \in \Reals$ a constant.
Then
    we have that
        $\Return'(\xi) = \Return(\xi) - k$
            for all possible and initial trajectories $\xi$,
        if and only if
            $\reward'$ is produced from $\reward$ by
                $k$-initial potential shaping and
                a mask of unreachable transitions.
\end{restatable}
\begin{proof}
The converse follows from \Cref{lemma:change_from_potentials} and that, by definition,
varying the reward for unreachable transitions does not affect the return
of any possible, initial trajectories.

For the forward direction, we show that the constant difference between $\Return'$ and
$\Return$ on possible initial trajectories implies a constant difference between the
returns of possible trajectories from any given reachable state, and that this 
state-dependent difference defines a $k$-initial potential function that transforms
$\reward$ into $\reward'$.

Consider an arbitrary reachable state $s \in \States$.
Let $\xi_s$ be some possible trajectory starting in $s$, and define
$\Delta_{\xi_s} = \Return(\xi_s) - \Return'(\xi_s)$,
the difference in return ascribed to this trajectory by $\Return$ and $\Return'$.
We show that $\Delta_{\xi_s}$ is independent of $\xi_s$ given $s$.
To extend $\xi_s$ into an initial trajectory, let $\zeta_s$ be some possible, initial,
trajectory fragment ending in $s$ (at least one exists, since $s$ is reachable; let its
length be $n$). Let $\zeta_s + \xi_s$ denote the concatenation of $\zeta_s$ and $\xi_s$.
Then,
\begin{align*}
\Delta_{\xi_s}
    &= \Return(\xi_s) - \Return'(\xi_s)                                \\
    \tag{$\dagger$}
    &= \frac{\Return(\zeta_s+\xi_s) - \Return(\zeta_s)}{\discount^n}
       - \frac{\Return'(\zeta_s+\xi_s) - \Return'(\zeta_s)}{\discount^n} \\
    \tag{$\ddagger$}
    &= \frac{k - \Return(\zeta_s) + \Return'(\zeta_s)}{\discount^n}\,.
\end{align*}
To reach ($\dagger$), note that by definition of return, 
$\Return(\zeta_s + \xi_s) = \Return(\zeta_s) + \discount^n \Return(\xi_s)$
(and likewise for $\Return'$), and recall that we have defined $\discount>0$.
To reach ($\ddagger$), note that since $\zeta_s + \xi_s$ is an initial trajectory,
we have by assumption $\Return(\zeta_s + \xi_s) - \Return'(\zeta_s + \xi_s) = k$.
Note ($\ddagger$) shows that $\Delta_{\xi_s}$ is independent of $\xi_s$ except
for a possible dependence on $\xi_s$'s starting state $s$ (arising through $\zeta_s$).

Thus, we may associate a unique $P(s) = \Delta_{\xi_s}$ with each reachable $s$.
Then $P(s)$ is a $k$-initial potential function on reachable states.
In particular, $P(s) = \Delta_{\xi_s} = k$ if $s$ is initial as then we may
choose $\zeta_s$ to be empty with $\Return(\zeta_s) = \Return'(\zeta_s) = 0$
and $n = 0$.
Furthermore, from the definition of terminal states we must have that $P(s) = \Delta_{\xi_s} = 0$
for terminal $s$.

Moreover, for reachable transitions, $\reward'$ is given by $k$-initial potential 
shaping of $\reward$ with $\Phi(s) = P(s)$.
Consider a reachable transition $(s, a, s')$. Let $\xi$ and $\xi'$ be possible 
trajectories such that $\xi = (s, a, s') + \xi'$. Then,
\begin{align*}
\reward(s,a,s') + \discount P(s') - P(s)
    &= \reward(s,a,s') + \discount(\Return(\xi') - \Return'(\xi'))
                - (\Return(\xi) - \Return'(\xi)) \\
    &= \Return'(\xi) - \discount \Return'(\xi') + (\reward(s,a,s') + \discount \Return(\xi')) - \Return(\xi) \\
    &= \Return'(\xi) - \discount \Return'(\xi') + \Return(\xi) - \Return(\xi) \\
    &= \Return'(\xi) - \discount \Return'(\xi') \\
    &= \reward'(s,a,s')\,.
\end{align*}
Any variation in reward for unreachable transitions can be accounted for by
a mask.
\end{proof}

\begin{restatable}{lemma}{thmPotentialsAndScaling}
\label{lemma:linear_scaling_of_G}
Let
  $(\States,\Actions,\TransitionDistribution,\InitStateDistribution,\reward,
  \discount)$ be an MDP,
  $\reward'$ a reward function, and
  $c \in \Reals$ a constant.
Then
        $\Return'(\xi) = c \cdot \Return(\xi)$
        for all possible initial trajectories $\xi$,
    if and only if
        $\reward'$ is produced from $\reward$ by
            zero-initial potential shaping,
            linear scaling by a factor of $c$, and
            a mask of all unreachable transitions.
\end{restatable}

\begin{proof}
    It is sufficient to show that the first condition is equivalent to 
    $\reward'$ being produced from $c \cdot \reward$ by
    zero-initial potential shaping and
    a mask of all unreachable transitions (in particular, any sequence of the 
    above three transformations from $\reward$ can be converted into a sequence
    where the linear scaling happens first).
    
    Denote by $\Return_c$ the return function of the scaled reward function
    $c \cdot \reward$.
    It is straightforward to show that
    $c \cdot \Return (\xi) = \Return_c(\xi)$ for all $\xi$.
    Then our first condition, $\Return'(\xi) = c\cdot\Return(\xi)$ for all
    possible initial trajectories,
    is equivalent to having $\Return'(\xi) = \Return_c(\xi)$ for these
    trajectories.
    
    By \Cref{lemma:potentials_and_episodes} (with $k=0$) this is equivalent
    to $\reward'$ being produced from $c \cdot \reward$ by zero-initial
    potential shaping and a mask of all unreachable transitions.
    This completes the proof.
\end{proof}

\newpage 

\section[Proofs for Section 3 Results]{Proofs for \Cref{sec:results} Results}
\label{apx:proofs}
\label{apx:proofs3}

We provide proofs for the theoretical results presented in the
main paper along with several general supporting lemmas from which these
results follow.

We distribute proofs of the results from \cref{sec:results-expert} across
three subsections.
\Cref{apx:proofs-q} proves results for the invariances of (soft)
$Q$-functions (\cref{lemma:Q_function_proximal_redistribution}).
\Cref{apx:proofs-entropy} proves results concerning alternative policies and
their trajectory distributions
(\cref{thm:boltzmann-rational,thm:boltzmann-rational-trajectories}).
\Cref{apx:results-optimal-expert} proves the results relating to optimal
policies and their trajectory distributions
(\cref{thm:optimal-policies-maximal,thm:optimal-trajectories}).

\subsection[Proofs for Section 3.1 Results Concerning Q-functions]{Proofs for
\Cref{sec:results-expert} Results Concerning $Q$-functions}
\label{apx:proofs-q}

\hypertarget{proof:q}{}
\thmQFunction*
\begin{proof}
$\Qfor{\pi}$ is the only function which satisfies the Bellman equation
(\ref{eq:bellman-pi})
for all $s\in\States$, $a\in\Actions$:
\[
\Qfor{\policy}(s,a)
=
\Expect{
    S' \sim \TransitionDistribution(s,a),
    A' \sim \policy(S')
}
{
    \reward(s,a,S') + \discount \cdot \Qfor{\policy}(S', A')
}.
\]
This equation can be rewritten as
\[
\Expect
    {S' \sim \TransitionDistribution(s,a)}
    {R(s,a,S')}
=
\Qfor{\policy}(s,a) - \gamma \cdot \Expect
    {S' \sim \TransitionDistribution(s,a), A' \sim \pi(S')} {\Qfor{\policy}(S', A')}.
\]
Since $\Qfor{\pi}$ is the only function which satisfies this equation for all
$s \in \States, a \in \Actions$, we have that the values of the left-hand
side for each $s \in \States, a \in \Actions$ together determine
$\Qfor{\pi}$, and vice versa. Since the left-hand side values are
preserved by $S'$-redistribution of $R$, and no other transformations
(cf.~\cref{def:sprime-redistribution}),
we have that $\Qfor{\pi}$ is preserved by $S'$-redistribution of $R$, and no
other transformations.

$\QStar = \Qfor{\OptimalPolicy}$ where $\OptimalPolicy$ is any optimal policy
derived from $\QStar$, so the invariances of $\QStar$ follow
as a special case.

The proof for $\QSoft$ is essentially the same.
$\QSoft$ is the only function that satisfies
\cref{eq:Q_soft} for all 
$s\in\States$, $a\in\Actions$:
\[
\QSoft(s,a) = 
    \Expect{S' \sim \TransitionDistribution(s,a)}
    {
        R(s,a,S') + \gamma \frac1\beta \log \sum_{a' \in \Actions} \exp \beta\QSoft (S', a')
    }\,.
\]
This can be rewritten as
    \[
\Expect{S' \sim \TransitionDistribution(s,a)} {R(s,a,S')}
=
\QSoft(s,a)
    -
    \gamma \cdot
    \Expect{S' \sim \TransitionDistribution(s,a)}
    {\frac1\beta \log \sum_{a' \in \Actions} \exp \beta\QSoft (S', a')}.
\]
Since $\QSoft$ is the only function which satisfies this
equation for all $s \in \States, a \in \Actions$, we have that the values
of the left-hand side for each $s \in \States, a \in \Actions$ together
determine $\QSoft$, and vice versa. Since the left-hand
side values are preserved by $S'$-redistribution of $R$, and no other
transformations
(cf.~\cref{def:sprime-redistribution}),
we have that $\QSoft$ is preserved by $S'$-redistribution of
$R$, and no other transformations.
\end{proof}

\subsection[Proofs for Results Concerning Alternative Policies]{Proofs for Results Concerning Alternative Policies}
\label{apx:proofs-entropy}

We split the proof of \cref{thm:boltzmann-rational} into two proofs, \cref{thm:boltzmann1} and \cref{thm:boltzmann2}, below.

In order to derive the invariances of the Boltzmann-rational policy, we
analyse a more general softmax-based policy we call a \emph{Boltzmann
policy}, of which the Boltzmann-rational policy is a special case.
Given a base policy $\basePolicy$,
and an \emph{inverse temperature} parameter $\beta>0$,
we define the \emph{Boltzmann policy} with respect to $\basePolicy$,
denoted
$\BoltzmannPolicyWrt{\basePolicy}$, 
using the softmax function:

\begin{equation}
\label{eq:boltzmann-policy}
\BoltzmannPolicyWrt{\basePolicy} (a \mid s)
= 
\frac
    {\exp\bigl(\beta \Afor{\basePolicy}(s, a)\bigr)}
    {\sum_{a' \in \Actions} \exp\bigl(\beta \Afor{\basePolicy}(s, a')\bigr)}.
\end{equation}

The \emph{Boltzmann-rational} policy, $\BoltzmannRationalPolicy$, is the
Boltzmann policy with respect to optimal policies (cf.
\ref{eq:boltzmann-rational-policy}).

We begin with \cref{lemma:advantage}, characterising the invariance of the advantage functions from which Boltzmann policies are derived.
This in turn supports  \Cref{lemma:boltzmann-general}, characterising the invariances of arbitrary Boltzmann policies.

\begin{restatable}{lemma}{thmAdvantage}
\label{lemma:advantage}
Given an MDP and a policy $\policy$,
    the advantage function for $\policy$, $\Afor{\policy}$,
determines $\reward$ up to
    $S'$-redistribution and potential shaping.
The optimal advantage $\AStar$ has the same invariances.
\end{restatable}
\begin{proof}
$\Afor{\policy}$ can be derived from $\Qfor{\policy}$, given $\policy$
    (by \cref{eq:bellman-pi},
    $
        \Afor{\policy}(s, a)
        = \Qfor{\policy}(s, a)
        -
        \Expect{A\sim\policy(s)}{\Qfor{\policy}(s, A)}
    $).
Thus $\Afor{\policy}$ is invariant to $S'$-redistribution
following \Cref{lemma:Q_function_proximal_redistribution}.
Moreover, by \Cref{lemma:change_from_potentials}, potential shaping causes
no change in $\Afor{\policy}$. That is, $\Afor{\policy}$ is also
invariant to potential shaping.

Conversely, let $\reward$ and $\reward'$ be such that
$\Afor{\policy} = \Afor{\policy}'$.
Define $\Phi : \States \to \Reals$ such that
$\Phi(s) = \Expect{A\sim\policy(s)}{\Qfor{\policy}(s,A) - \Qfor{\policy}'(s,A)}$.
This $\Phi$ satisfies the requirements of a potential function
(all $\Q$-values from terminal states are zero).
Potential shaping $\reward$ with $\Phi$ yields a new reward function,
denoted $\reward^{(\Phi)}$,
with $Q$-function denoted $\Qfor{\policy}^{(\Phi)}$.
Then observe, for each $s\in\States, a\in\Actions$:
\begin{align*}
\Qfor{\policy}^{(\Phi)}
    &= \Qfor{\policy}(s,a) - \Phi(s)
    & \text{(by \Cref{lemma:change_from_potentials})}
\\  &=
    (\Qfor{\policy}(s,a)
        - \Expect{A\sim\policy(s)}{\Qfor{\policy}(s,A)})
        + \Expect{A\sim\policy(s)}{\Qfor{\policy}'(s,A)}
\\  &=
    \Afor{\policy}(s,a)
        + \Expect{A\sim\policy(s)}{\Qfor{\policy}'(s,A)}
\\  &=
    \Afor{\policy}'(s,a)
        + \Expect{A\sim\policy(s)}{\Qfor{\policy}'(s,A)}
    & \text{($\Afor{\policy} = \Afor{\policy}'$ by assumption)}
\\  &=
    (\Qfor{\policy}'(s,a)
        - \Expect{A\sim\policy(s)}{\Qfor{\policy}'(s,A)})
        + \Expect{A\sim\policy(s)}{\Qfor{\policy}'(s,A)}
\\  &=
    \Qfor{\policy}'(s,a)\,.
\end{align*}
That is, $\reward^{(\Phi)}$ and $\reward'$ share a $Q$-function. Thus,
by \Cref{lemma:Q_function_proximal_redistribution},
$\reward'$ is given by $S'$-redistribution from $\reward^{(\Phi)}$.

The optimal advantage function's invariances arise as a special case, since
$\AStar = \Afor{\OptimalPolicy}$,
where $\OptimalPolicy$ is any optimal policy derived from $\AStar$.
\end{proof}

\begin{lemma}
\label{lemma:boltzmann-general}
Given an MDP, an inverse temperature parameter $\beta$, and a base policy $\basePolicy$,
the Boltzmann policy $\BoltzmannPolicyWrt{\basePolicy}$
determines $\reward$ up to
$S'$-redistribution and potential shaping.
\end{lemma}

\begin{proof}
By \cref{eq:boltzmann-policy}, $\BoltzmannPolicyWrt{\basePolicy}$ can be
derived from $\Afor{\basePolicy}$.
Thus $\BoltzmannPolicyWrt{\basePolicy}$ is invariant to $S'$-redistribution
and potential shaping by \Cref{lemma:advantage}.

Conversely, we show that $\Afor{\basePolicy}$ can be derived from 
$\BoltzmannPolicyWrt{\basePolicy}$ in turn.
Therefore $\BoltzmannPolicyWrt{\basePolicy}$ can have no more invariances
than $\Afor{\basePolicy}$, amounting to $S'$-redistribution and potential
shaping by \Cref{lemma:advantage}.

For each $s\in\States, a\in\Actions$, observe:
\begin{align*}
    \BoltzmannPolicyWrt{\basePolicy} (a \mid s)
    &= 
    \frac
        {\exp\bigl(\beta \Afor{\basePolicy}(s, a)\bigr)}
        {\sum_{a' \in \Actions} \exp\bigl(\beta \Afor{\basePolicy}(s, a')\bigr)}
\\
\tag{$\dagger$}
\rightarrow\quad
    \Afor{\basePolicy}(s,a)
    &=
    \frac1\beta\log\BoltzmannPolicyWrt{\basePolicy} (a \mid s)
    +
    \frac1\beta\log
        {\sum_{a' \in \Actions} \exp\bigl(\beta \Afor{\basePolicy}(s, a')\bigr)}
    \,.
\end{align*}
We have not yet solved for $\Afor{\basePolicy}$, since it still occurs on
both sides of ($\dagger$).
However, we can eliminate the RHS occurrence by appealing to the following
identity (that the advantage has zero mean in each state $s\in\States$):
\begin{align*}
    \Expect{A\sim\basePolicy(s)}{\Afor{\basePolicy}(s,A)}
    &=
    \Expect{A\sim\basePolicy(s)}{
        \Qfor{\basePolicy}(s,A)
        -
        \Expect{A'\sim\basePolicy(s)}{\Qfor{\basePolicy}(s,A')}
    }
\\  &=
    \Expect{A\sim\basePolicy(s)}{\Qfor{\basePolicy}(s,A)}
    -
    \Expect{A\sim\basePolicy(s)}{\Qfor{\basePolicy}(s,A)}
\\  &=
    0\,.
\end{align*}
Taking the expectation of both side of ($\dagger$) therefore yields:
\begin{align*}
    \Expect{A\sim\basePolicy(s)}{
        \Afor{\basePolicy}(s,A)
    }
    &=
    \Expect{A\sim\basePolicy(s)}{
    \frac1\beta\log\BoltzmannPolicyWrt{\basePolicy} (A \mid s)
    }
    +
    \Expect{A\sim\basePolicy(s)}{
    \frac1\beta\log
        {\sum_{a' \in \Actions} \exp\bigl(\beta \Afor{\basePolicy}(s, a')\bigr)}
    }
\\
\rightarrow \quad
    0
    &=
    \Expect{A\sim\basePolicy(s)}{
    \frac1\beta\log\BoltzmannPolicyWrt{\basePolicy} (A \mid s)
    }
    +
    \frac1\beta\log
        {\sum_{a' \in \Actions} \exp\bigl(\beta \Afor{\basePolicy}(s, a')\bigr)}
\\\tag{$\ddagger$}
\rightarrow \quad
    \frac1\beta\log
        {\sum_{a' \in \Actions} \exp\bigl(\beta \Afor{\basePolicy}(s, a')\bigr)}
    &=
    - \Expect{A\sim\basePolicy(s)}{
    \frac1\beta\log\BoltzmannPolicyWrt{\basePolicy} (A \mid s)
    }
    \,.
\end{align*}
Combining ($\ddagger$) with ($\dagger$) gives us an expression for
$\Afor{\basePolicy}$ in terms only of $\BoltzmannPolicyWrt{\basePolicy}$,
as required:
\[
    \Afor{\basePolicy}(s,a)
    =
    \frac1\beta\log\BoltzmannPolicyWrt{\basePolicy} (a \mid s)
    - \Expect{A\sim\basePolicy(s)}{
        \frac1\beta\log\BoltzmannPolicyWrt{\basePolicy} (A \mid s)
    }
    \,.
\]

\end{proof}

Our main result concerning Boltzmann-rational policies follows immediately.
\hypertarget{proof:boltzmann-rational-policy}{}
\begin{restatable}{theorem}{thmBoltzmannPolicy}
\label{thm:boltzmann1}
Given an MDP and an inverse temperature parameter $\beta$, the Boltzmann-rational policy $\BoltzmannRationalPolicy$,
determines $\reward$ up to
$S'$-redistribution and potential shaping.
\end{restatable}
\begin{proof}
The Boltzmann-rational policy $\BoltzmannRationalPolicy$ determines its own base policy $\OptimalPolicy$.
This is because the maximum probability actions in $\BoltzmannRationalPolicy$ 
are precisely those actions with maximal optimal advantage $\AStar$
($\argmax_{a\in\Actions}\AStar(s) = \argmax_{a\in\Actions} \BoltzmannRationalPolicy(s,a)$).
We can break ties arbitrarily, as any optimal base policy will lead to the same Boltzmann-rational policy.
So, given $\BoltzmannRationalPolicy$, we are effectively also given a base policy, and the invariances of $\BoltzmannRationalPolicy$ therefore follow as a special case of
\cref{lemma:boltzmann-general}.
\end{proof}

We turn to prove the corresponding result about MCE policies, which follows a
similar line of reasoning relative to the soft $Q$-function.
We use an elementary property of the softmax function, which we state and
derive as \Cref{lemma:softmax} for the convenience of the unfamiliar reader.
\hypertarget{proof:mce-policy}{}
\begin{restatable}{theorem}{thmMCEPolicy}
\label{thm:boltzmann2}
Given an MDP and an inverse temperature $\beta$, the MCE policy $\MCEPolicy$
determines $\reward$ up to
$S'$-redistribution and potential shaping.
\end{restatable}
\begin{proof}
$\MCEPolicy$ is given by applying the softmax function to
$\QSoft$.
Recall (or see \Cref{lemma:softmax}, below) that the softmax function is
invariant to a constant shift, and no other transformations.
This means that
$\MCEPolicy$ is invariant to exactly those transformations
that induce constant shifts in $\QSoft$ for each state.

$S'$-redistribution induces no shift in $\QSoft$ by
\Cref{lemma:Q_function_proximal_redistribution}. By \cref{lemma:potentials_to_soft_Q}, potential shaping
induces a state-dependent constant shift.
Thus, $\MCEPolicy$ is invariant to $S'$-redistribution and potential shaping.

Conversely, we show that any state-dependent constant shift in $\QSoft$
can be described by these two kinds of transformations. Therefore, they
are the only invariances.
Let $B : \States \rightarrow \Reals$, and suppose $R_1$ and $R_2$
are two reward functions such that the corresponding soft $Q$-functions
satisfy
$\QSoftN{1}(s,a) = \QSoftN{2}(s,a) + B(s)$.
Then,
\begin{align*}
    \Expect{S' \sim \TransitionDistribution(s,a)}{R_1(s,a,S')}
    &= \QSoftN{1}(s,a)
        - \Expect{S' \sim \TransitionDistribution(s,a)}{
            \gamma \frac1\beta
            \log \sum_{a' \in A} \exp \beta\QSoftN{1} (S', a')
        }
\\
    &= \QSoftN{2}(s,a) + B(s)
        - \Expect{S' \sim \TransitionDistribution(s,a)}{
            \gamma \frac1\beta
            \log \sum_{a' \in A} \exp \beta\left(
                \QSoftN{2} (S', a') + B(S')
            \right)
        }
\\
    &= \QSoftN{2}(s,a) + B(s)
        - \Expect{S' \sim \TransitionDistribution(s,a)}{
            \gamma \frac1\beta \log \left(
                \sum_{a' \in A} \exp \beta\QSoftN{2} (S', a')
            \right)
            + \gamma B(S')
        }
\\
    &= \Expect{S' \sim \TransitionDistribution(s,a)}{
            R_2(s,a,S') + B(s) - \gamma B(S')
        }\,.
\end{align*}
Now set $\Phi(s) = -B(s)$, and we can see that the difference between $R$ and $R'$ is described by potential shaping and $S'$-redistribution.
\end{proof}

\begin{restatable}{lemma}{thmSoftmaxConstant}
\label{lemma:softmax}
Consider two functions $f : \mathcal{X} \to \Reals$ and
$g : \mathcal{X} \to \Reals$
defined on a finite set $\mathcal{X}$.
Then the softmax distributions over $f$ and $f+g$ agree, that is,
for all $x\in\mathcal{X}$,
\begin{equation}
    \label{eq:softmax-plus-identity}
    \frac{\exp(f(x)+g(x))}{\sum_{x'\in\mathcal{X}}\exp(f(x')+g(x'))}
    =
    \frac{\exp(f(x))}{\sum_{x'\in\mathcal{X}}\exp(f(x'))}
    \,,
\end{equation}
    if and only if $g$ is a constant function over $\mathcal{X}$.
\end{restatable}
\begin{proof}
This is an elementary property of the softmax function.
The forward direction can be seen by manipulating
\cref{eq:softmax-plus-identity} as follows:
\begin{align*}
    \frac{\exp(f(x)+g(x))}{{\exp(f(x))}}
    &=
    \frac{\sum_{x'\in\mathcal{X}}\exp(f(x')+g(x'))}{\sum_{x'\in\mathcal{X}}\exp(f(x'))}
\\  \rightarrow\quad
    g(x)
    &=
    \log\left(\frac{\sum_{x'\in\mathcal{X}}\exp(f(x')+g(x'))}{\sum_{x'\in\mathcal{X}}\exp(f(x'))}\right)
\end{align*}
which is constant in $x$.

The converse can be seen as follows. Assume $g(x) = G$, a constant.
Then,
\begin{equation*}
    \frac{\exp(f(x)+g(x))}{\sum_{x'\in\mathcal{X}}\exp(f(x')+g(x'))}
    =
    \frac
        {\exp(f(x))\cdot\exp(G)}
        {\left(\sum_{x'\in\mathcal{X}}\exp(f(x'))\right)\cdot\exp(G)}
    =
    \frac{\exp(f(x))}{\sum_{x'\in\mathcal{X}}\exp(f(x'))}\,.
\end{equation*}
\end{proof}

We continue with results about the trajectories derived from these
alternative policies.
We once again prove \cref{thm:boltzmann-rational-trajectories} in two parts
(\cref{thm:boltzmannt1,thm:boltzmannt2}).
For Boltzmann-rational trajectories, we once again
provide a more general lemma concerning arbitrary Boltzmann policies
(\ref{eq:boltzmann-policy}).
\begin{restatable}{lemma}{thmBoltzmannTrajectories}
\label{lemma:boltzmann-trajectories-general}
Given an MDP $M$, an inverse temperature $\beta$, and a base policy
$\basePolicy$,
the distribution of trajectories,
    $\TrajDistributionBoltzmannWrt{\basePolicy}$,
induced by the Boltzmann policy $\BoltzmannPolicyWrt{\basePolicy}$ acting
in MDP $M$
determines $\reward$ up to
$S'$-redistribution, potential shaping, and a mask of \emph{unreachable} transitions.
\end{restatable}
\begin{proof}
That the distribution is invariant to $S'$-redistribution and potential
shaping follows from \Cref{lemma:boltzmann-general}.
The distribution is also invariant to changes in the reward for transitions out
of unreachable states, since these rewards cannot affect the policy for reachable
states. As a result, the distribution is additionally invariant to a mask of unreachable
transitions.

The trajectory distribution can be factored into the separate distributions 
$\BoltzmannPolicyWrt{\basePolicy}(s) \in \Delta(\Actions)$ for each
reachable state $s$ by conditioning on a supported prefix trajectory fragment
that leads to $s$ and marginalising over subsequent states and actions.
Via a similar argument to the proof of \Cref{lemma:boltzmann-general}, the
distribution determines the reward function for transitions (out of these
reachable states) up to potential shaping and 
$S'$-redistribution (as they affect reachable states).
\end{proof}

\hypertarget{proof:boltzmann-rational-trajectories}{}
\begin{restatable}{theorem}{thmBoltzmannRationalTrajectories}
\label{thm:boltzmannt1}
Given an MDP $M$ and an inverse temperature parameter $\beta$,
the distribution of trajectories, $\TrajDistributionBoltzmann$,
induced by the Boltzmann-rational policy 
$\BoltzmannRationalPolicy$ acting in MDP $M$,
determines $\reward$ up to
$S'$-redistribution, potential shaping, and a mask of unreachable
transitions.
\end{restatable}
\begin{proof}
As in \Cref{thm:boltzmann-rational}, the invariances for the
Boltzmann-rational policy's trajectories arises as a special case.
\end{proof}

We turn to prove the corresponding result about MCE policies, which follows a
similar line of reasoning as for the Boltzmann trajectories, but relative to
the MCE policy instead.
\hypertarget{proof:mce-trajectories}{}
\begin{restatable}{theorem}{thmMCETrajectories}
\label{thm:boltzmannt2}
Given an MDP $M$ and an inverse temperature parameter $\beta$,
the distribution of trajectories, $\TrajDistributionMCE$,
induced by the MCE policy $\MCEPolicy$ acting in MDP $M$
determines $\reward$ up to
$S'$-redistribution, potential shaping, and a mask of unreachable
transitions.
\end{restatable}
\begin{proof}
Directly analogous to the proof of
\Cref{lemma:boltzmann-trajectories-general},
(relative to \Cref{thm:boltzmann2}).
\end{proof}

\subsection[Proofs for Results Concerning Optimal Policies]{Proofs for Results Concerning Optimal Policies}
\label{apx:results-optimal-expert}

Our results concerning the invariance of optimal policies and their trajectories follow from the following general result connecting general optimality-preserving transformations to the set of optimal actions in some subset of states.

The key idea of the proof is to establish a link between the value-bounding function $\Psi$ (\cref{def:opt-preserving}) and the optimal value function for $\reward'$ via the Bellman optimality equation. We note that the definition of (general) optimality-preserving transformations is designed specifically to elicit this link.

\begin{restatable}{lemma}{thmOptimalPolicyGeneral}
\label{thm:optimal-policies}
Given an MDP $M$, suppose we have the set of optimal actions
for each state in a subset of states $\mathfrak{S} \subseteq \States$.
Let $\mathfrak{O}$ be the set of (set-valued) functions
    $\OPTFunc : \States \rightarrow \Powerset(\Actions)\setminus\{\emptyset\}$ 
such that $\OPTFunc(s) = \argmax_{a\in\Actions} \AStar(s,a)$ for all $s \in \mathfrak{S}$ (but where
$\OPTFunc$ is unconstrained outside $\mathfrak{S}$).
Then, these optimal action sets determine $\reward$ up to general optimality-preserving
transformations with $\OPTFunc \in \mathfrak{O}$.
\end{restatable}
\begin{proof}
Suppose $\reward'$ is obtained from $M$'s reward $\reward$ via a general
optimality-preserving transformation with some $\OPTFunc \in \mathfrak{O}$.
Let $\Psi$ be the corresponding value-bounding function, that is, a
function $\Psi : \States \to \Reals$ satisfying, for all $s \in \States$ and
$a\in\Actions$,
\begin{equation}
\label{eq:condition-on-psi}
    \Expect{S' \sim \TransitionDistribution(s,a)}{R'(s, a, S') + \discount \cdot \Psi(S')}
    \leq
    \Psi(s)
\,,
\end{equation}
with equality if and only if $a \in \OPTFunc(s)$.
Since $\OPTFunc(s)$ is nonempty (by definition), we have for all $s \in \States$
\begin{equation*}
    \Psi(s)
    =
    \max_{a\in\Actions}
    \left(
        \Expect
            {S' \sim \TransitionDistribution(s,a)}
            {R'(s, a, S') + \discount \cdot \Psi(S')}    
    \right)
\,.
\end{equation*}
This recursive condition on $\Psi$ is the Bellman optimality equation for the
unique optimal value function, $\VStar'$, of the MDP with transformed reward
$\reward'$.
Therefore, $\Psi(s) = \VStar'(s)$ for all $s \in \States$, and we can
rewrite \cref{eq:condition-on-psi} as 
\begin{equation}
\label{eq:condition-as-vstar}
    \Expect{S' \sim \TransitionDistribution(s,a)}{R'(s, a, S') + \discount \cdot \VStar'(S')}
    \leq
    \VStar'(s)
\,,
\end{equation}
with equality only for $a \in \OPTFunc(s)$.

Now, consider a state $s \in \mathfrak{S}$.
By assumption, for this $s$, $\OPTFunc(s) = \argmax_{a\in\Actions}\AStar(s, a)$.
Then for this state, the actions that attain the optimal value bound in
\cref{eq:condition-as-vstar} are these same optimal actions.
Therefore, $\reward'$ induces the same sets of optimal actions from states
in $\mathcal{S}$.

Conversely,
consider a second MDP $M'$, differing from $M$ only in its reward
function, $\reward'$. Assume the set of optimal actions in states in $\mathfrak{S}$
agrees with the optimal actions in $M$ for those states.
Let $\VStar'$ and $\AStar'$ denote the optimal value and advantage
functions for $M'$.
The Bellman optimality equation for $M'$ ensures that, for $s\in\States$, 
\begin{equation}
\label{eq:bellman-gives-condition}
    \VStar'(s)
    =
    \max_{a\in\Actions}
    \left(
        \Expect
            {S' \sim \TransitionDistribution(s,a)}
            {R'(s, a, S') + \discount \cdot \VStar'(S')}
    \right)
\end{equation}
with the maximum attained precisely by the actions
$a \in \argmax_{a\in\Actions}(\AStar'(s, a))$.
Setting $\OPTFunc(s) = \argmax_{a\in\Actions}(\AStar'(s, a))$,
\cref{eq:bellman-gives-condition} can be rewritten as 
\begin{equation}
\label{eq:vstar-is-psi}
    \Expect
        {S' \sim \TransitionDistribution(s,a)}
        {R'(s, a, S') + \discount \cdot \VStar(S')}
    \leq
    \VStar(s)
\end{equation}
for all $s\in\States$ and $a\in\Actions$, with equality if and only if $a \in \OPTFunc(s)$.

Now, for $s\in\mathfrak{S}$, we have $\argmax_{a\in\Actions}(\AStar'(s, a)) =
\argmax_{a\in\Actions}(\AStar(s, a))$, 
because $M$ and $M'$ have matching sets of optimal actions for these states
(by assumption).
Then, \cref{eq:vstar-is-psi} shows that $\reward'$ is produced from $\reward$ by
a general optimality-preserving transformation with $\OPTFunc(s) = \argmax_{a\in\Actions}(\AStar'(s, a))$
(and $\Psi(s) = \VStar'(s)$).
\end{proof}

We are now in a position to prove
\Cref{thm:optimal-policies-maximal,thm:optimal-trajectories}:

\hypertarget{proof:supportive-optimal-policies}{}
\thmOptimalPoliciesMaximal*
\begin{proof}
By assumption, our optimal policies are maximally supportive. Therefore,
their support determines the set of optimal actions from all states.
Also by assumption, our maximally supportive optimal policies are determined
by the set of optimal actions in each state.
Therefore, a maximally supportive optimal policy has the same reward
information as the set of optimal policies in each state. Its invariances
follow as a special case of \Cref{thm:optimal-policies}, with
$\mathfrak{S}=\States$.
\end{proof}

\hypertarget{proof:supportive-optimal-trajectories}{}
\thmOptimalTrajectories*
\begin{proof}
The distribution of trajectories can be factored into separate distributions
$\OptimalPolicy(s) \in \Delta(\Actions)$ for each state $s\in\mathfrak{S}$
(in a manner similar to \Cref{lemma:boltzmann-trajectories-general}, as proved above).
As above, these individual distributions determine and are determined by the
set of optimal actions within each of those states.
The invariance result therefore follows from \Cref{thm:optimal-policies}.
\end{proof}

\begin{remark}\label{remark:optimal-policy-assumptions}
As mentioned in \cref{sec:results}, when there are multiple optimal policies,
invariances depend on how the given policy is chosen. The proofs above
reveal that our assumptions are crucial in connecting maximally supportive
optimal policies to optimal action sets. We comment on the motivation for
these assumptions and, following our theme of cataloguing partial
identifiability, we sketch how the result would change without them.

Assumption (1), that the given policy is maximally supportive, allows us to
rule out unsupported actions as suboptimal.
Additional reward transformations could become permissible otherwise.
As a well-known example, the zero reward function is consistent with any
policy if unsupported actions could also be optimal \citep{ng2000}.
This more general case is difficult to analyse within our framework,
because it is not well-described by transformations or an equivalence
relation.
The assumption may be demanding,
but the consequences of misspecification are mild in the case of policy
optimisation -- at least the learnt reward function won't allow any
suboptimal actions to become optimal.

Assumption (2), that a given policy is computed only from the set of optimal
actions in each state, appears to be common. The purpose of this technical
assumption is to rule out pathological schemes for encoding additional reward
information through the selection of the policy.
Through such schemes one could in principle encode the full reward function,
for example into the infinite decimal representation of the probability of
taking one action over another in some state.
Such a selection scheme, even if it was not known to the learner, would
remove invariances, as transformations that change the reward function but
not the set of optimal states would change the given policy.
\end{remark}

\subsection[Proofs for Section 3.3 Results]{Proofs for \Cref{sec:results-eval} Results}
\label{apx:proofs3_2}

\begin{restatable}{theorem}{thmReturnFrag}
\label{thm:labels_of_fragments}
Given an MDP,
the return function restricted to possible trajectory fragments, $\Return_\zeta$,
determines $\reward$ up to
a mask of \emph{impossible} transitions.
\end{restatable}
\begin{proof}
The result is immediate, since the restricted domain still includes
all possible transitions (as length one trajectory fragments with return
equal to the reward of the transition), and no fragments with impossible
transitions.
\end{proof}

\hypertarget{proof:return-trajectories}{}
\thmReturnTraj*
\begin{proof}
The result follows from \Cref{lemma:potentials_and_episodes}
with $k=0$.
\end{proof}

\hypertarget{proof:soft-compare-fragments}{}
\thmSoftComparisonsFrag*
\begin{proof}
Since $\cmpSoftFrag$ can be derived from $\Return_\zeta$, it is invariant to
a mask of impossible transitions by \Cref{thm:labels_of_fragments}.
Conversely, $\cmpSoftFrag$ determines $\reward$ for all \emph{possible}
transitions.
This is because $R(s,a,s')$ is encoded in the Boltzmann distribution of
comparisons between
the length zero trajectory fragment $\zeta_0 = (s)$ and
the length one trajectory fragment $\zeta_1 = (s, a, s')$,
and can be recovered as follows:
\begin{align*}
\Probability(\zeta_0 \cmpSoftFrag \zeta_1)
    &=  \frac
        {\exp(\beta\Return(\zeta_1))}
        {\exp(\beta\Return(\zeta_0)) + \exp(\beta\Return(\zeta_1))}
    \\&=   \frac
        {\exp(\beta\reward(s, a, s'))}
        {\exp(\beta\cdot0) + \exp(\beta\reward(s, a, s'))}
\\
\rightarrow\quad
    \reward(s, a, s')
    &=
    \frac1\beta\cdot\log\left(
        \frac
            {\Probability(\zeta_0 \cmpSoftFrag \zeta_1)}
            {1-\Probability(\zeta_0 \cmpSoftFrag \zeta_1)}
    \right)
\,.
\end{align*}
Therefore $\cmpSoftFrag$ is invariant to precisely a mask of impossible transitions.
\end{proof}

\hypertarget{proof:soft-compare-trajectories}{}
\thmSoftComparisonsTraj*
\begin{proof}
Note that as $\cmpSoftTraj$ can be derived from $\Return_\xi$,
by \Cref{thm:labels_of_trajectories},
$\cmpSoftTraj$ is invariant to zero-initial potential shaping and a mask of
unreachable transitions.
It is additionally invariant to $k$-initial potential shaping for arbitrary 
constants $k\in\Reals$, and no other transformations:
$\Return_\xi$ can be recovered from $\cmpSoftTraj$ up to a constant
(we can compare all possible initial trajectories to an arbitrary reference
trajectory and recover their relative return using a similar manipulation as
above, but we can't determine the return of the reference trajectory).
From there, the precise invariance follows from \Cref{lemma:potentials_and_episodes}.
\end{proof}

\hypertarget{proof:hard-compare-fragments}{}
\thmHardComparisonOfFragments*
\begin{proof}
For (1), positive linear scaling of reward by a constant $c$ leads to the same
scaling of the return of each trajectory fragment, and this always preserves
the relation $\cmpStarFrag$, since for any $c > 0$,
$
    c \cdot \Return(\zeta_1) \leq c \cdot \Return(\zeta_2)
    \Leftrightarrow
    \Return(\zeta_1) \leq \Return(\zeta_2)
$
for all pairs of trajectory fragments $\zeta_1, \zeta_2$.
Moreover, $\cmpStarFrag$ inherits invariance to a mask of impossible transitions
from $\Return_\zeta$ (\Cref{thm:labels_of_fragments}).

For (2), let $\reward'$ be produced from $\reward$ via some transformation that is
\emph{neither} a mask of impossible transitions \emph{nor} a zero-preserving monotonic
transformation.
It must be that either $\reward'$ fails to preserve the ordinal comparison of two possible
transitions, or that it fails to preserve the set of zero-reward possible transitions,
compared to $\reward$.
In the first case, consider two possible transitions whose rewards are not preserved,
$x_1$ and $x_2$.
Without loss of generality, suppose $\reward(x_1) \leq \reward(x_2)$ but
$\reward'(x_1) > \reward'(x_2)$.
This corresponds to a change in $\cmpStarFrag$'s comparison of the length one
trajectories formed from $x_1$ and $x_2$, namely $x_1 \cmpStarFrag x_2$ from true
to false.
Similarly, in the second case, the comparisons between the transition whose
reward became or ceased to be zero and a length one trajectory (with return 0)
will have changed.
Therefore, $\cmpStarFrag$ is not invariant to such transformations.

The bound (1) is attained by the following MDP invariant precisely to
positive linear scaling and a mask of impossible transitions.
Let $\States = \{s\}$, $\Actions = \{a_1, a_2\}$, $\reward(s, a_1, s) = 1$,
and $\reward(s, a_2, s) = 1 + \discount$.
Since $\reward(s, a_2, s) = \reward(s, a_1, s) + \discount \reward(s, a_1, s)$, the corresponding order
relation will contain both $(s, a_2, s) \cmpStarFrag (s, a_1, s, a_1, s)$ and 
$(s, a_1, s, a_1, s) \cmpStarFrag (s, a_2, s)$.
This property requires that $\reward(s,a_1,s) = (1 + \gamma) \cdot \reward(s,a_2,s)$,
which is preserved only by linear scaling of $\reward$.
(Non-positive linear scaling is already ruled out by (2)).

The bound (2) is attained by the following MDP invariant to arbitrary zero-preserving
monotonic transformations.
Let $\States = \{s_1, s_2\}$, $\Actions = \{a\}$, with possible transitions
$(s_1, a, s_2)$ and $(s_2, a, s_2)$, and $\reward(s_1, a, s_2) > \reward(s_2, a, s_2) = 0$.
Any zero-preserving monotonic transformation of $\reward$ preserves the ordering of
all \emph{possible} trajectory fragments, namely that all nonempty trajectories starting
in $s_1$ have positive return and all other possible trajectories have zero return.
\end{proof}

\hypertarget{proof:hard-compare-trajectories-bound}{}
\thmHardComparisonOfTrajectoriesBound*
\begin{proof}
The pairwise Boltzmann distributions of $\cmpSoftTraj$ can be used to derive
the noiseless comparisons of $\cmpStarTraj$, since the relative return of
each pair of trajectories is encoded in each
    $\Probability(\xi_1 \cmpSoftFrag \xi_2)$:
\[
\xi_1 \cmpStarTraj \xi_2
\Leftrightarrow
    \bigl(
    \Return(\xi_1) \leq \Return(\xi_2)
    \bigr)
\Leftrightarrow
    \bigl(
    \exp(\beta\Return(\xi_1)) \leq \exp(\beta\Return(\xi_2))
    \bigr)
\Leftrightarrow
    \left(
    \tfrac12 \leq \Probability(\xi_1 \cmpSoftFrag \xi_2)
    \right)
    .
\]
Therefore, $\cmpStarTraj$ is invariant to $k$-initial potential shaping and
a mask of unreachable transitions by
\Cref{thm:boltzmann-comparisons-trajectories}.

That $\cmpStarTraj$ is also invariant to positive linear scaling follows from
a similar argument as for the first bound in
\Cref{thm:noiseless-comparisons-trajectory-fragments}, proved above.
\end{proof}
\begin{remark}
\label{remark:noiseless-comparisons-trajectories-bound}
\Cref{thm:hard-trajectory-comparison} is a lower bound on the
full set of invariances of the noiseless order of possible and initial
trajectories. We note the following:
\begin{itemize}
    \item
        It is not a tight bound: At least in some MDPs, the order is
        invariant to additional transformations.
    \item
        A slightly tighter lower bound can be achieved by establishing that
        $\cmpStarTraj$ can be derived from $\cmpStarFrag$:
        Consider, for a given trajectory $\xi$, the sequence of `prefix'
        trajectory fragments $\xi^{(0)}, \xi^{(1)}, \xi^{(2)}, \ldots$, with
        each $\xi^{(n)}$ comprising the first $n$ transitions of $\xi$. By
        definition $\Return(\xi) =
        \lim_{n\to\infty}\Return(\xi^{(n)})$, and so for each pair
        of trajectories $\xi_1, \xi_2$, we have $\xi_1 \cmpStarTraj \xi_2$ if
        and only if $\xi_1^{(n)} \cmpStarFrag \xi_2^{(n)}$ for infinitely
        many $n$.  While this is not a practical method to compute the
        trajectory order $\cmpStarTraj$ from the fragment order
        $\cmpStarFrag$, it counts as a derivation in that it is sufficient to
        show that if a transformation does not change the fragment order
        $\cmpStarFrag$, it cannot change the trajectory order $\cmpStarTraj$
        either. Therefore, in particular, $\cmpStarTraj$ inherits invariance
        to ZPMTs in \emph{some} MDPs from $\cmpStarFrag$. This tightens the
        bound, at least in some MDPs.
    \item
        The previous point does not imply that the trajectory order
        $\cmpStarTraj$ inherits the fragment order $\cmpStarFrag$'s
        \emph{non}-invariances.  A case in point is that $\cmpStarTraj$ is
        invariant to $k$-initial potential shaping and a mask of unreachable
        transitions, where $\cmpStarFrag$ is not
        (\Cref{thm:noiseless-comparisons-trajectory-fragments}).  It is not
        yet clear if there are MDPs where $\cmpStarTraj$ is invariant to no
        ZPMTs other than positive linear scaling, or even to not all ZPMTs,
        and there may be invariances of $\cmpStarTraj$ that require new
        transformation classes to describe.
        However, \Cref{thm:hard-trajectory-comparison} and this
        remark give us enough information to confidently position
        $\cmpStarTraj$ in our partial order of reward-derived objects.
\end{itemize}
\end{remark}

\hypertarget{proof:compare-lotteries}{}
\thmLotteries*
\begin{proof}
It is clear that preferences between lotteries over a choice set are
preserved by positive affine transformations of the value (and no other
transformations). In particular, the converse is a consequence of the
well-known VNM utility theorem \citep{vnm2}.
The proof by \citet{vnm2} covers a finite number of outcomes, and the result
also holds for an infinite number of
outcomes~\citep[see, e.g.,][]{fishburn1970}.

Thus, our result is immediate from
\Cref{lemma:potentials_and_episodes,lemma:linear_scaling_of_G}, which
together state that these positive affine transformations of the return
function correspond exactly to $k$-initial potential shaping, positive linear
scaling, and a mask of unreachable transitions.
\end{proof}

\newpage

\section[Proofs for Section 4 Results]{Proofs for \Cref{sec:more} Results}
\label{apx:proofs4}

\thmComplementaryAmbiguity*
\begin{proof}
Transformations that preserve $(X, Y)$ necessarily preserve $X$, therefore
$(X, Y) \refines X$.
But since $X$ and $Y$ are incomparable, there is some transformation that
preserves $X$ and not $Y$. This transformation does not preserve $(X, Y)$.
Therefore, $(X, Y) \refinesStrict X$.
Similarly, $(X, Y) \refinesStrict Y$.
\end{proof}
We note that the above result is also an elementary consequence of the 
\emph{lattice structure} of the partial order of partition refinement
\citep[\S I.2.B]{aigner1979}, since the combined data source corresponds
to the \emph{meet} of the original data sources.

\thmSRedistributionAndTransfer*
\begin{proof}
Per \cref{def:sprime-redistribution}, 
that $\reward_2$ is produced from $\reward_1$ by $S'$-redistribution under
$\TransitionDistribution$ requires that, for all $s\in\States$ and $a\in\Actions$,
\begin{equation}
\label{eq:thm:sr-transfer:condition-1}
    \Expect{S' \sim \TransitionDistribution(s,a)}{\reward_1(s,a,S')}
    =
    \Expect{S' \sim \TransitionDistribution(s,a)}{\reward_2(s,a,S')}
\,.
\end{equation}
Let $s\in\States$ and $a\in\Actions$ be any state and action such that $\TransitionDistribution'(s, a) \neq \TransitionDistribution(s, a)$.
Let $\Vec{\TransitionDistribution}_{s,a}$ and
$\Vec{\TransitionDistribution'}_{s,a}$ be
$\TransitionDistribution(s, a)$ and $\TransitionDistribution'(s, a)$
expressed as vectors,
and let $\Vec{R_1}_{s,a}$ be the vector where
$\Vec{R_1}_{s,a}^{(i)} = R_1(s,a,s_i)$.
The question is then if there is an analogous vector
$\Vec{R_2}_{s,a}$ such that:
\begin{equation}
\label{eq:thm:sr-transfer:condition-2}
\begin{aligned}
    \Vec{\TransitionDistribution}_{s,a} \cdot \Vec{R_2}_{s,a}
    &=
    \Vec{\TransitionDistribution}_{s,a} \cdot \Vec{R_1}_{s,a}\,,
\\
    \Vec{\TransitionDistribution'}_{s,a} \cdot \Vec{R_2}_{s,a}
    &= \mathcal{L}(s,a).
\end{aligned}
\end{equation}
Since $\Vec{\TransitionDistribution}_{s,a}$ and
$\Vec{\TransitionDistribution'}_{s,a}$ differ and are valid probability
distributions, they are linearly independent.
Therefore, the system of equations (\ref{eq:thm:sr-transfer:condition-2})
always has a solution for $\Vec{R_2}_{s,a}$.
Form the required $\reward_2$ as $\reward_1$ modified to have the values of
$\Vec{\reward_2}_{s,a}$ in these states where the transition function is
disturbed.
\end{proof}

\newpage

\section{Other Spaces of Reward Functions}
\label{apx:rewards}

Hitherto, we have assumed reward functions are members of $\SxAxS \to \Reals$.
That is, they are deterministic functions of transitions depending on the state, action, and successor state.
In this appendix, we discuss several alternative spaces of reward functions and
their implications for the invariance properties of various objects derived from
the reward function.

\subsection{Restricted-domain Reward Functions}

It is common in both reinforcement learning and reward learning to consider less
expressive spaces of reward functions.
In particular, the domain of the reward function is often restricted to $\States$
or $\SxA$.
When modelling a task, the choice of reward function domain is usually a
formality:
An MDP taking full advantage of the domain $\SxAxS$ has an ``equivalent''
MDP with a restricted domain and some added auxiliary states~\citep[\S17]{aima3e}.
Conversely, reward functions with restricted domains can be viewed as a
special case of functions from $\SxAxS$ where the functions are constant in
the final argument(s).
Restricting the domain can be an appealing simplification when modelling a task,
hence the popularity of these formulations.

When modelling a data source, this equivalence may not apply: We may not have
access to data regarding auxiliary states, so assuming a restricted domain
effectively assumes the latent reward is indeed constant with respect to the
successor state (and possibly the action) of each transition.
This assumption may or may not be warranted.

If a restricted domain of $\States$ or $\SxA$ is preferred, then our invariance
results can be adapted in a straightforward manner.
In general, since we are effectively considering a subspace of candidate reward
functions for transformations, ambiguity can only decrease.
In particular, these restrictions have two main consequences.

Firstly, the reward function transformation of $S'$-redistribution vanishes
to the identity transformation, since it allows variation only in the successor state
argument of the reward function, which is now impossible.
This reduces the effective ambiguity of the $Q$-function and all derivative objects.
Notably, the $Q$-function uniquely identifies the reward function, and Boltzmann
policies have the same invariances as Boltzmann comparisons between trajectories.
Restricting the domain to $\States$ means the (state) value function for an
arbitrary known policy also uniquely identifies the reward function
but doesn't otherwise alter the invariances we have explored.

Secondly, for most MDPs, the available potential-shaping transformations are
restricted, but not eliminated.
The function added in a potential-shaping transformation
($\discount \cdot \Phi(s') - \Phi(s)$)
nominally depends on the successor state of the transition.
Some transformed reward functions may rely on this dependence, falling outside
of the restricted domain.
However, some non-zero transformations will usually remain.
For example, in a discounted MDP without terminal states, a non-zero constant
potential function $\Phi(s) = k$ does not \emph{effectively} depend on $s$,
and the reward transformation of adding
$\discount \cdot \Phi(s') - \Phi(s) = (\discount - 1)\cdot k$ to a reward function
does not introduce a dependence on $s'$.
In general, the set of remaining potential-shaping transformations will depend
on the network structure of the MDP.
At the extreme, in a deterministic MDP with state-action rewards, all potential-shaping transformations are permitted, since a dependence on $s'$ can be satisfied
by $a$.

\subsection{Stochastic Reward Functions}

Certain tasks are naturally modelled as providing rewards drawn
stochastically from some distribution upon each transition.
An even more expressive space of reward functions than we consider is the space of \emph{transition-conditional reward distributions}.\footnotemark{}
Identifying the reward function in this case is more challenging in general
because the latent parameter contains a full distribution of information for
each input, rather than a single point.
In the spirit of this paper, we sketch a characterisation of this
additional ambiguity.

\footnotetext{
    Of course, it's also possible to consider reward to be distributed
    conditionally on only the state or state-action components of a transition
    and not the full transition.
}

A deterministic reward function can be viewed as the conditional expectation
of a reward distribution function.
Taking the expectation of the reward distribution for each transition
introduces invariance, since the expectation operation is not injective
(except in certain restricted cases such as for parametric families of 
distributions that can be parametrised by their mean).
The invariance introduced is akin to $S'$-redistribution, but with an
expectation over the support of the reward distribution rather than the
successor state of each transition.

In the extension of the RL formalism to account for stochastic rewards,
this expectation is effectively the first step in the derivation of each
of the objects we have studied.
Therefore, all of these objects inherit this new invariance.

As a consequence, all data sources are effectively more ambiguous with
respect to this new latent parameter.
For example, if optimal comparisons between trajectories are understood
to be performed based on the pairwise comparison of the expected return
of each individual trajectory, then these comparisons are also invariant
to transformations of the reward distributions that preserve their means.

Fortunately, much of reinforcement learning also focuses on expected
return and reward \emph{in application}.
Accordingly, most downstream tasks are tolerant to any ambiguity in the
exact distribution of stochastic rewards, beyond identifying the mean.
Since this is the same kind of ambiguity that is introduced by considering
the latent parameter of reward learning as a conditional distribution rather
than a deterministic function, our results are still informative for these
situations.

\subsection{Further Spaces and Future Work}

For certain applications, including \emph{risk-sensitive RL} where
non-mean objectives are
pursued~\citep{morimura2010nonpara, morimura2010para, dabney2018},
the distribution of stochastic rewards can be consequential.
Moreover, the introduction of stochastic rewards suggests considering
data sources based on samples rather than expectations, such as a data
source of trajectory comparisons based on sampled trajectory returns.
Characterising the invariances of these objectives to transformations
of the reward distribution, and thereby their \emph{ambiguity tolerance},
is left to future work.

In future extensions of this work to handle continuous MDPs, there will be
an opportunity to study the effect of restricting to various parametrised
spaces of reward functions.
For example, it is common in reinforcement learning and reward learning to
study MDPs with reward functions that are linear in a \emph{feature vector}
associated with each transition.
This kind of restriction may reduce the available reward transformations
compared to those available to a non-parametric reward function in a similar
manner to restricting the domain of a finite reward function as discussed
above.

The relaxation of the Markovian assumption also introduces a broader space
of reward functions and with it new dimensions for transformations and
invariance.
As one example related to potential shaping, the non-Markovian
additive transformations studied by \citet{wiewiora2003} will amount to new
invariances of the optimal policy and other related objects.

\end{document}